\documentclass[letterpaper]{article}
\usepackage{algorithm}
\usepackage{algpseudocode}
\usepackage{amsmath}
\usepackage{amssymb}
\usepackage{bbm}
\usepackage{bm}
\usepackage{color}
\usepackage{dsfont}
\usepackage{graphicx}
\usepackage{jmlr2e}
\usepackage{subfigure}
\usepackage{times}

\newcommand{\commentout}[1]{}
\newcommand{\junk}[1]{}

\newcommand{\greedy}{{\tt Greedy}}
\newcommand{\opm}{{\tt OPM}}

\newcommand{\bw}{{\bf w}}

\newcommand{\bx}{{\bf x}}

\newcommand{\by}{{\bf y}}

\newcommand{\cI}{\mathcal{I}}

\newcommand{\eps}{\varepsilon}

\newcommand{\realset}{\mathbb{R}}

\newcommand{\abs}[1]{\left|#1\right|}

\newcommand{\E}[2]{\mathbb{E}_{#2} \! \left[#1\right]}

\newcommand{\floors}[1]{\left\lfloor#1\right\rfloor}
\newcommand{\I}[1]{\mathds{1} \! \left\{#1\right\}}

\newcommand{\set}[1]{\left\{#1\right\}}

\DeclareMathOperator*{\argmax}{arg\,max\,}
\DeclareMathOperator*{\argmin}{arg\,min\,}
 
\begin{document}

\title{Learning to Act Greedily: Polymatroid Semi-Bandits}

\author{\name Branislav Kveton
\email branislav.kveton@technicolor.com \\
\addr Technicolor Labs \\
175 S San Antonio Road, Suite 200, Los Altos, CA 94022, USA \AND
\name Zheng Wen
\email zhengwen@yahoo-inc.com \\
\addr Yahoo Labs \\
701 1st Ave, Sunnyvale, CA 94089, United States \AND
\name Azin Ashkan
\email azin.ashkan@technicolor.com \\
\addr Technicolor Labs \\
175 S San Antonio Road, Suite 200, Los Altos, CA 94022, USA \AND
\name Michal Valko
\email michal.valko@inria.fr \\
\addr INRIA Lille - Nord Europe, SequeL team \\
40 avenue Halley 59650, Villeneuve d'Ascq, France}

\editor{}

\maketitle
 
\begin{abstract}
Many important optimization problems, such as the minimum spanning tree and minimum-cost flow, can be solved optimally by a greedy method. In this work, we study a learning variant of these problems, where the model of the problem is unknown and has to be learned by interacting repeatedly with the environment in the bandit setting. We formalize our learning problem quite generally, as learning how to maximize an unknown \emph{modular function} on a known \emph{polymatroid}. We propose a computationally efficient algorithm for solving our problem and bound its expected cumulative regret. Our gap-dependent upper bound is tight up to a constant and our gap-free upper bound is tight up to polylogarithmic factors. Finally, we evaluate our method on three problems and demonstrate that it is practical.
\end{abstract}



\section{Introduction}
\label{sec:introduction}

Many important combinatorial optimization problems, such as the minimum-cost flow \citep{megiddo74optimal} and minimum spanning tree \citep{papadimitriou98combinatorial}, can be solved optimally by a greedy algorithm. These problems can be solved efficiently because they can be viewed as optimization on \emph{matroids} \citep{whitney35abstract} or \emph{polymatroids} \citep{edmonds70submodular}. More specifically, they can be formulated as finding the maximum of a modular function on the polytope of a submodular function. In this work, we study a learning variant of this problem where the modular function is unknown.

Our learning problem is sequential and divided into episodes. In each episode, the learning agent chooses a feasible solution to our problem, the basis of a polymatroid; observes noisy weights of all items with non-zero contributions in the basis; and receives the dot product between the basis and the weights as a payoff. The goal of the learning agent is to maximize its expected cumulative return over time, or equivalently to minimize its expected cumulative regret. Many practical problems can be formulated in our setting, such as learning a routing network \citep{oliveira05survey} where the delays on the links of the network are stochastic and initially unknown. In this problem, the bases are spanning trees, the observed weights are the the delays on the links of the spanning tree, and the cost is the sum of the observed delays.

This paper makes three contributions. First, we bring together the concepts of bandits \citep{lai85asymptotically,auer02finitetime} and polymatroids \citep{edmonds70submodular}, and propose \emph{polymatroid bandits}, a new class of stochastic learning problems. A multi-armed bandit \citep{lai85asymptotically} is a framework for solving online learning problems that require exploration. The framework has been successfully applied to a variety of problems, including those in combinatorial optimization \citep{gai12combinatorial,cesabianchi12combinatorial,audibert14regret}. In this paper, we extend bandits to the combinatorial optimization problems that can be solved greedily.

Second, we propose a simple algorithm for solving our problem, which explores based on the optimism in the face of uncertainty. We refer to our algorithm as \emph{Optimistic Polymatroid Maximization ($\opm$)}. $\opm$ has two key properties. First, it is computationally efficient because the basis in each episode is chosen greedily. Second, $\opm$ is also sample efficient. In particular, we derive a gap-dependent upper bound on the expected cumulative regret of $\opm$ and show that it is tight up to a constant, and we also derive a gap-free upper bound and show that it is tight up to polylogarithmic factors. Our upper bounds exploit the structural properties of polymatroids and improve over general-purpose bounds for stochastic combinatorial semi-bandits.

Finally, we evaluate $\opm$ on three problems. The first problem is a synthetic flow network and we use it to demonstrate that our gap-dependent upper bound is quite practical, an order of magnitude larger than the observed regret. The second problem is learning of a routing network for an Internet service provider (ISP). The last problem is learning how to recommend diverse movies. All three problems can be solved efficiently in our framework. This demonstrates that $\opm$ is practical and can solve a wide range of problems.

We adopt the following notation. We write $A + e$ instead of $A \cup \set{e}$, and $A + B$ instead of $A \cup B$. We also write $A - e$ instead of $A \setminus \set{e}$, and $A - B$ instead of $A \setminus B$.

\section{Polymatroids}
\label{sec:polymatroids}

In this section, we first introduce polymatroids and then illustrate them on practical problems. A \emph{polymatroid} \citep{edmonds70submodular} is a polytope associated with a submodular function. More specifically, a polymatroid is a pair $M = (E, f)$, where $E = \set{1, \dots, L}$ is a \emph{ground set} of $L$ items and $f: 2^E \to \realset^+$ is a function from the power set of $E$ to non-negative real numbers. The function $f$ is \emph{monotonic}, $\forall X \subseteq Y \subseteq E: f(X) \leq f(Y)$; \emph{submodular}, $\forall X, Y \subseteq E: f(X) + f(Y) \geq f (X \cup Y) + f(X \cap Y)$; and $f(\emptyset) = 0$. Since $f$ is monotonic, $f(E)$ is one of its maxima. We refer to $f(E)$ as the \emph{rank} of a polymatroid and denote it by $K$. Without loss of generality, we assume that $f(e) \leq 1$ for all items $e \in E$. Because $f$ is submodular, we indirectly assume that $f(X + e) - f(X) \leq 1$ for all $X \subseteq E$.

The \emph{independence polyhedron} $P_M$ associated with polymatroid $M$ is a subset of $\realset^L$ defined as:
\begin{align}
  \textstyle
  P_M = \set{\bx: \bx \in \realset^L, \ \bx \geq 0,
  \ \forall X \subseteq E: \sum_{e \in X} \bx(e) \leq f(X)},
  \label{eq:independence polyhedron}
\end{align}
where $\bx(e)$ is the $e$-th entry of vector $\bx$. The vector $\bx$ is \emph{independent} if $\bx \in P_M$. The \emph{base polyhedron} $B_M$ is a subset of $P_M$ defined as:
\begin{align}
  \textstyle
  B_M = \set{\bx: \bx \in P_M, \ \sum_{e \in E} \bx(e) = K}.
  \label{eq:base polyhedron}
\end{align}
The vector $\bx$ is a \emph{basis} if $\bx \in B_M$. In other words, $\bx$ is independent and its entries sum up to $K$.

\subsection{Optimization on Polymatroids}
\label{sec:optimization}

\begin{algorithm}[t]
  \caption{$\greedy$: Edmond's algorithm for computing the maximum-weight basis of a polymatroid.}
  \label{alg:greedy}
  \begin{algorithmic}
    \State {\bf Input:} Polymatroid $M = (E, f)$, weights $\bw$
    \State
    \State Let $e_1, \dots, e_L$ be an ordering of items such that:
    \State \quad $\bw(e_1) \geq \ldots \geq \bw (e_L)$
    \State $\bx \gets \text{All-zeros vector of length } L$
    \ForAll{$i = 1, \dots, L$}
      \State $\bx(e_i) \gets f(\set{e_1, \dots, e_i}) - f(\set{e_1, \dots, e_{i  - 1}})$
    \EndFor
    \State
    \State {\bf Output:} Maximum-weight basis $\bx$
  \end{algorithmic}
\end{algorithm}

A \emph{weighted polymatroid} is a polymatroid associated with a vector of weights $\bw \in (\realset^+)^L$. The $e$-th entry of $\bw$, $\bw(e)$, is the weight of item $e$. A classical problem in polyhedral optimization is to find a \emph{maximum-weight basis} of a polymatroid:
\begin{align}
  \bx^\ast = \argmax_{\bx \in B_M} \langle\bw, \bx\rangle = \argmax_{\bx \in P_M} \langle\bw, \bx\rangle.
  \label{eq:optimal}
\end{align}
This basis can be computed greedily (Algorithm~\ref{alg:greedy}). The greedy algorithm works as follows. First, the items $E$ are sorted in decreasing order of their weights, $\bw(e_1) \geq \ldots \geq \bw (e_L)$. We assume that the ties are broken by an arbitrary but fixed rule. Second, $\bx^\ast$ is computed as $\bx^\ast(e_i) = f(\set{e_1, \dots, e_i}) - f(\set{e_1, \dots, e_{i  - 1}})$ for all $i$. Note that the \emph{minimum-weight basis} of a polymatroid with weights $\bw$ is the maximum-weight basis of the same polymatroid with weights $\max_{e \in E} \bw(e) - \bw$:
\begin{align}
  \argmin_{\bx \in B_M} \langle\bw, \bx\rangle =
  \argmax_{\bx \in B_M} \langle\max_{e \in E} \bw(e) - \bw, \bx\rangle.
  \label{eq:minimum basis}
\end{align}
So the minimization problem is mathematically equivalent to the maximization problem \eqref{eq:optimal}, and all results in this paper straightforwardly generalize to the minimization.

Many existing problems can be viewed as optimization on a polymatroid \eqref{eq:optimal}. For instance, polymatroids generalize \emph{matroids} \citep{whitney35abstract}, a notion of independence in combinatorial optimization that is closely related to computational efficiency. In particular, let $M = (E, \cI)$ be a matroid, where $E = \set{1, \dots, L}$ is its ground set, $\cI \subseteq 2^E$ are its independent sets, and:
\begin{align}
  f(X) = \max_{Y: Y \subseteq X, Y \in \cI} \abs{Y}
\end{align}
is its \emph{rank function}. Let $\bw \in (\realset^+)^L$ be a vector of non-negative weights. Then the maximum-weight basis of a matroid:
\begin{align}
  A^\ast = \argmax_{A \in \cI} \sum_{e \in A} \bw(e)
\end{align}
can be also derived as $A^\ast = \set{e: \bx^\ast(e) = 1}$, where $\bx^\ast$ is the maximum-weight basis of the corresponding polymatroid. The basis is $\bx^\ast \in \set{0, 1}^L$ because the rank function is a monotonic submodular function with zero-one increments \citep{fujishige05submodular}.

Our optimization problem can be written as a \emph{linear program (LP)} \citep{bertsimas97introduction}:
\begin{align}
  \max_\bx \sum_{e \in E} \bw(e) \bx(e) \qquad
  \text{subject to:} \sum_{e \in X} \bx(e) \leq f(X) \quad
  \forall X \subseteq E,
  \label{eq:LP}
\end{align}
where $\bx \in (\realset^+)^L$ is a vector of $L$ optimized variables. This LP has exponentially many constraints, one for each subset $X \subseteq E$. Therefore, it cannot be solved directly. Nevertheless, $\greedy$ can solve the problem in $O(L \log L)$ time. Therefore, our problem is a very efficient form of linear programming.

Many combinatorial optimization concepts, such as \emph{flows} and \emph{entropy} \citep{fujishige05submodular}, are submodular. Therefore, optimization on these concepts involves polymatroids. A well-known problem in this class is the \emph{minimum-cost flow} \citep{megiddo74optimal}. This problem can be formulated as follows. The ground set $E$ are the source nodes of a flow network, $f(X)$ is the maximum flow through source nodes $X \subseteq E$, and $\bw(e)$ is the cost of a unit flow through source node $e$. The minimum-weight basis of this polymatroid is the maximum flow with the minimum cost \citep{fujishige05submodular}, which we refer to as the minimum-cost flow.

The problem of recommending diverse items can be also cast as optimization on a polymatroid \citep{ashkan14diversified,ashkan14dum}. Let $E$ be a set of recommendable items, $f(X)$ be the number of topics covered by items $X$, and $\bw$ be a weight vector such that $\bw(e)$ is the popularity of item $e$. Then $\bx^\ast = \greedy(M, \bw)$ is a vector such that $\bx^\ast(e) > 0$ if and only if item $e$ is the most popular item in at least one topic covered by item $e$. We illustrate this concept on a simple example. Let the ground set $E$ be a set of $3$ movies:
\begin{center}
  \begin{tabular}{clrl} \hline
    $e$ & Movie title & Popularity $\bw(e)$ & Movie genres \\ \hline
    1 & Inception & 0.8 & Action \\
    2 & Grown Ups 2 & 0.5 & Comedy \\
    3 & Kindergarten Cop & 0.6 & Action Comedy \\ \hline
  \end{tabular}
\end{center}
Let $f(X)$ be the number of movie genres covered by movies $X$. Then $f$ is submodular and defined as:
\begin{alignat}{8}
  f(\emptyset) & = 0, & \qquad
  f(\set{2}) & = 1, & \qquad
  f(\set{1, 2}) & = 2, & \qquad
  f(\set{2, 3}) & = 2, \\
  f(\set{1}) & = 1, & \qquad
  f(\set{3}) & = 2, & \qquad
  f(\set{1, 3}) & = 2, & \qquad
  f(\set{1, 2, 3}) & = 2. \nonumber
\end{alignat}
The maximum-weight basis of polymatroid $M = (E, f)$ is $\bx^\ast = (1, 0, 1)$, and $\set{e: \bx^\ast(e) > 0} = \set{1, 3}$ is the minimal set of movies that cover each movie genre by the most popular movie in that genre.

\subsection{Combinatorial Optimization on Polymatroids}
\label{sec:combinatorial optimization}

In this paper, we restrict our attention to the feasible solutions:
\begin{align}
  \Theta = \set{\bx: \left(\exists \bw \in (\realset^+)^L: \bx = \greedy(M, \bw)\right)}
  \label{eq:feasible set}
\end{align}
that can be computed greedily for some weight vector $\bw$ and define our objective as finding:
\begin{align}
  \bx^\ast = \argmax_{\bx \in \Theta} \langle\bw, \bx\rangle.
  \label{eq:optimal basis}
\end{align}
The set $\Theta$ are the vertices of $B_M$ and we prove this formally in Lemma~\ref{lem:greedy vertices} in Appendix.

Our choice is motivated by three reasons. First, we study the problem of learning to act greedily. So we are only interested in the bases that can be computed greedily. Second, many optimization problems of our interest (Section~\ref{sec:optimization}) are combinatorial in nature and only the bases in $\Theta$ are suitable feasible solutions. For instance, in a graphic matroid, $\Theta$ is a set of spanning trees. In a linear matroid, $\Theta$ is a set of maximal sets of linearly independent vectors. The bases in $B_M - \Theta$ do not have this interpretation. Another example is our recommendations problem in Section~\ref{sec:optimization}. In this problem, for any $\bx = \greedy(M, \bw)$, $\set{e: \bx(e) > 0}$ is a minimal set of items that cover each topic by the most popular item according to $\bw$. The bases in $B_M - \Theta$ cannot be interpreted in this way. Finally, we note that our choice does not have any impact on the notion of optimality. In particular, let $\bx$ be optimal for some $\bw$. Then $\bx^g = \greedy(M, \bw)$ is also optimal and since $\bx^g \in \Theta$, it follows that:
\begin{align}
  \max_{\bx \in B_M} \langle\bw, \bx\rangle =
  \max_{\bx \in \Theta} \langle\bw, \bx\rangle.
\end{align}


\section{Polymatroid Bandits}
\label{sec:polymatroid bandits}

The maximum-weight basis of a polymatroid cannot be computed when the weights $\bw$ are unknown. This may happen in practice. For instance, suppose that we want to recommend a diverse set of popular movies (Section~\ref{sec:optimization}) but the popularity of these movies is initially unknown, perhaps because the movies are newly released. In this work, we study a learning variant of maximizing a modular function on a polymatroid that can solve this type of problems.


\subsection{Model}
\label{sec:model}

We formalize our learning problem as a polymatroid bandit. A \emph{polymatroid bandit} is a pair $(M, P)$, where $M$ is a polymatroid and $P$ is a probability distribution over the weights $\bw \in \realset^L$ of items $E$ in $M$. The $e$-th entry of $\bw$, $\bw(e)$, is the weight of item $e$. We assume that the weights $\bw$ are drawn i.i.d. from $P$ and that $P$ is unknown. Without loss of generality, we assume that $P$ is a distribution over the unit cube $[0, 1]^L$. Other than that, we do not assume anything about $P$. We denote the expected weights of the items by $\bar{\bw} = \mathbb{E}[\bw]$. By our assumptions on $P$, $\bar{\bw}(e) \geq 0$ for all items $e$.

Each item $e$ is associated with an \emph{arm} and each feasible solution $\bx \in \Theta$ is associated with a set of arms $A = \set{e: \bx(e) > 0}$. The arms $A$ are the items with non-zero contributions in $\bx$. After the arms are \emph{pulled}, the learning agent receives a \emph{payoff} of $\langle\bw, \bx\rangle$ and \emph{observes} $\set{(e, \bw(e)): \bx(e) > 0}$, the weights of all items with non-zero contributions in $\bx$. This feedback model is known as \emph{semi-bandit} \citep{audibert14regret}. The solution to our problem is a maximum-weight basis in expectation:
\begin{align}
  \bx^\ast =
  \arg\max_{\bx \in \Theta} \E{\langle\bw, \bx\rangle}{\bw} =
  \arg\max_{\bx \in \Theta} \langle\bar{\bw}, \bx\rangle.
  \label{eq:optimal arm}
\end{align}
This problem is equivalent to problem \eqref{eq:optimal basis} and therefore can be solved greedily, $\bx^\ast = \greedy(M, \bar{\bw})$.

We choose our observation model for several reasons. First, the model is a natural generalization of that in matroid bandits \citep{kveton14matroid}. In matroid bandits, the bases are of the form $\bx \in \set{0, 1}^L$ and the learning agents observes the weights of all chosen items $e$, $\bx(e) = 1$. In this case, $\bx(e) = 1$ is equivalent to $\bx(e) > 0$. Second, our observation model is suitable for our motivating examples (Section~\ref{sec:optimization}). Specifically, in the minimum-cost flow problem, we assume that the learning agent observes the costs of all source nodes that contribute to the maximum flow. In the movie recommendation problem, the agent observes individual movies chosen by the user, from a set of recommended movies. Finally, our observation model allows us to derive similar regret bounds to those in matroid bandits \citep{kveton14matroid}.

Our learning problem is \emph{episodic}. Let $(\bw_t)_{t = 1}^n$ be an i.i.d. sequence of weights drawn from distribution $P$. In episode $t$, the learning agent chooses basis $\bx_t$ based on its prior actions $\bx_1, \dots, \bx_{t - 1}$ and observations of $\bw_1, \dots, \bw_{t - 1}$; gains $\langle\bw_t, \bx_t\rangle$; and observes $\set{(e, \bw_t(e)): \bx_t(e) > 0}$, the weights of all items with non-zero contributions in $\bx_t$. The agent interacts with the environment in $n$ episodes. The goal of the agent is to maximize its expected cumulative return, or equivalently to minimize its \emph{expected cumulative regret}:
\begin{align}
  R(n) = \E{\sum_{t = 1}^n R(\bx_t, \bw_t)}{\bw_1, \dots, \bw_n},
  \label{eq:cumulative regret}
\end{align}
where $R(\bx, \bw) = \langle\bw, \bx^\ast\rangle - \langle\bw, \bx\rangle$ is the regret associated with basis $\bx$ and weights $\bw$.


\subsection{Algorithm}
\label{sec:algorithm}

\begin{algorithm}[t]
  \caption{$\opm$: Optimistic polymatroid maximization.}
  \label{alg:bandit}
  \begin{algorithmic}
    \State {\bf Input:} Polymatroid $M = (E, f)$
    \State
    \State Observe $\bw_0 \sim P$
    \Comment{Initialization}
    \State $\hat{\bw}_1(e) \gets \bw_0(e) \hspace{1.795in} \forall e \in E$
    \State $T_0(e) \gets 1 \hspace{2.095in} \forall e \in E$
    \State
    \ForAll{$t = 1, \dots, n$}
      \State $U_t(e) \gets \hat{\bw}_{T_{t - 1}(e)}(e) + c_{t - 1, T_{t - 1}(e)} \hspace{0.48in} \forall e \in E$
      \Comment{Compute UCBs}
      \State $\bx_t \gets \greedy(M, U_t)$
      \Comment{Find a maximum-weight basis}
      \State Observe $\set{(e, \bw_t(e)): \bx_t(e) > 0}$, where $\bw_t \sim P$
      \Comment{Choose the basis}
      \State
      \State $T_t(e) \gets T_{t - 1}(e) \hspace{1.525in} \forall e \in E$
      \Comment{Update statistics}
      \State $T_t(e) \gets T_t(e) + 1 \hspace{1.43in} \forall e: \bx_t(e) > 0$
      \State $\displaystyle \hat{\bw}_{T_t(e)}(e) \gets
      \frac{T_{t - 1}(e) \hat{\bw}_{T_{t - 1}(e)}(e) + \bw_t(e)}{T_t(e)}
      \hspace{0.2in} \forall e: \bx_t(e) > 0$
    \EndFor
  \end{algorithmic}
\end{algorithm}

Our learning algorithm is designed based on the \emph{optimism in the face of uncertainty} principle \citep{auer02finitetime}. In particular, it is a greedy method for finding a maximum-weight basis of a polymatroid where the expected weight $\bar{\bw}(e)$ of each item is substituted with its optimistic estimate $U_t(e)$. We refer to our method as \emph{Optimistic Polymatroid Maximization ($\opm$)}.

The pseudocode of $\opm$ is given in Algorithm~\ref{alg:bandit}. In each episode $t$, the algorithm works as follows. First, we compute an \emph{upper confidence bound} (UCB) on the expected weight of each item $e$:
\begin{align}
  U_t(e) = \hat{\bw}_{T_{t - 1}(e)}(e) + c_{t - 1, T_{t - 1}(e)},
  \label{eq:UCB}
\end{align}
where $\hat{\bw}_{T_{t - 1}(e)}(e)$ is our estimate of the expected weight $\bar{\bw}(e)$ in episode $t$, $c_{t - 1, T_{t - 1}(e)}$ is the radius of the confidence interval around this estimate, and $T_{t - 1}(e)$ denotes the number of times that item $e$ is selected in the first $t - 1$ episodes, $\bx_i(e) > 0$ for $i < t$. Second, we compute the maximum-weight basis with respect to $U_t$ using $\greedy$. Finally, we select the basis, observe the weights of all items $e$ where $\bx_t(e) > 0$, and then update our model $\hat{\bw}$ of the environment. The radius:
\begin{align}
  c_{t, s} = \sqrt{\frac{2 \log t}{s}}
  \label{eq:confidence radius}
\end{align}
is designed such that each UCB is a high-probability upper bound on the corresponding weight $\hat{\bw}_s(e)$. The UCBs encourage exploration of items that have not been observed sufficiently often. As the number of past episodes increases, we get better estimates of the weights $\bar{\bw}$, all confidence intervals shrink, and $\opm$ starts exploiting most rewarding items. The $\log(t)$ term increases with time and enforces continuous exploration.

For simplicity of exposition, we assume that $\opm$ is initialized by observing each item once. In practice, this initialization step can be implemented efficiently in the first $L$ episodes. In particular, in episode $t \leq L$, $\opm$ chooses first item $t$ and then all other items, in an arbitrary order. The corresponding regret is bounded by $K L$ because $\langle\bar{\bw}, \bx\rangle \in [0, K]$ for any $\bar{\bw}$ (Section~\ref{sec:model}) and basis $\bx$ (Section~\ref{sec:polymatroids}).

$\opm$ is a greedy method and therefore is extremely computationally efficient. In particular, suppose that the function $f$ is an oracle that can be queried in $O(1)$ time. Then the time complexity of $\opm$ in episode $t$ is $O(L \log L)$, comparable to that of sorting $L$ numbers. The design of $\opm$ is not very surprising and it draws on prior work \citep{kveton14matroid,gai12combinatorial}.

Our major contribution is that we derive a tight upper bound on the regret of $\opm$. Our analysis is novel and is a significant improvement over \cite{kveton14matroid}, who analyze the regret of $\opm$ in the context of matroids. Roughly speaking, the analysis of \cite{kveton14matroid} leverages the augmentation property of a matroid. Our analysis is based on the submodularity of a polymatroid.


\section{Analysis}
\label{sec:analysis}

This section is organized as follows. First, we propose a novel decomposition of the regret of $\opm$ in a single episode (Section~\ref{sec:regret decomposition}). Loosely speaking, we decompose the regret as a sum of its parts, the fractional gains of individual items in the optimal and suboptimal bases. This part of the proof relies heavily on the structure of a polymatroid and is a major contribution. Second, we apply the regret decomposition to bound the regret of $\opm$ (Section~\ref{sec:upper bounds}). Third, we compare our regret bounds to existing upper bounds (Section~\ref{sec:upper bound improvement}) and prove matching lower bounds (Section~\ref{sec:lower bounds}). Finally, we summarize our results (Section~\ref{sec:discussion}).

\subsection{Regret Decomposition}
\label{sec:regret decomposition}

The key step in our analysis is that we bound the expected regret in episode $t$ for any basis $\bx_t$, $R(\bx_t, \bar{\bw})$. In rest of this section, we fix the basis $\bx_t$ and drop indexing by time $t$ to simplify our notation.

Without loss of generality, we assume that the items in the ground set $E$ are ordered such that $\bar{\bw}(1) \geq \ldots \geq \bar{\bw}(L)$. So the optimal basis $\bx^\ast$ is defined as:
\begin{align}
  \bx^\ast(i) = f(A^\ast_i) - f(A^\ast_{i - 1}) \qquad i = 1, \dots, L;
\end{align}
where $A^\ast_i = \set{1, \dots, i}$ are the first $i$ items in $E$. Let $U_t$ be the vector of UCBs in episode $t$ and $a_1, \dots, a_L$ be the ordering of items such that $U_t(a_1) \geq \ldots \geq U_t(a_L)$. Then the basis $\bx$ in episode $t$ is defined as:
\begin{align}
  \bx(a_k) = f(A_k) - f(A_{k - 1}) \qquad k = 1, \dots, L;
\end{align}
where $A_k = \set{a_1, \dots, a_k}$. The hardness of discriminating items $e$ and $e^\ast$ is measured by a \emph{gap} between the expected weights of the items:
\begin{align}
  \Delta_{e, e^\ast} = \bar{\bw}(e^\ast) - \bar{\bw}(e).
  \label{eq:gap}
\end{align}
For each item $e$, we define $\rho(e)$, the largest index such that $\bar{\bw}(\rho(e)) > \bar{\bw}(e)$ and $\bx^\ast(\rho(e)) > 0$, the expected weight of item $\rho(e)$ is larger than that of item $e$ and the item contributes to $\bx^\ast$. For simplicity of exposition, we assume that item $1$ contributes to the optimal basis $\bx^\ast$, $\bx^\ast(1) > 0$. This guarantees that $\rho(e)$ is properly defined for all items but item $1$. We assume that $\rho(1) = 0$.

Our regret decomposition is based on rewriting the difference in the expected returns of bases $\bx^\ast$ and $\bx$ as the sum of the differences in the returns of intermediate solutions, which are obtained by interleaving the bases. We refer to these solutions as \emph{augmentations}. A \emph{$k$-augmentation} is a vector $\by_k \in [0, 1]^L$ such that:
\begin{align}
  \by_k(i) = \left\{
  \begin{array}{ll@{\quad}}
  f(A_j) - f(A_{j - 1}) & i \in A_k \text{ and } a_j = i \\
  f(A_k + A^\ast_i) - f(A_k + A^\ast_{i - 1}) & i \notin A_k
  \end{array}
  \right|
  \qquad i = 1, \dots, L.
  \label{eq:augmentation}
\end{align}
It can be also viewed as a basis generated by $\greedy$, which first selects $k$ suboptimal items $a_1, \dots, a_k$ and then the remaining $L - K$ items, ordered from $1$ to $L$. Now we prove our first lemma.

\begin{lemma}
\label{lem:gain difference} For any $k$, the difference of two consecutive augmentations $\by_{k - 1}$ and $\by_k$ satisfies:
\begin{align*}
  \by_{k - 1}(i) - \by_k(i) \ \ \left\{
  \begin{array}{ll@{\quad}}
  = 0 & i \in A_{k - 1} \\
  \leq 0 & i = a_k \\
  \geq 0 & i \notin A_k
  \end{array}
  \right|
  \qquad i = 1, \dots, L.
\end{align*}
\end{lemma}
\begin{proof}
First, let $i = a_j \in A_{k - 1}$. Then by definition \eqref{eq:augmentation}:
\begin{align}
  \by_{k - 1}(i) - \by_k(i) = f(A_j) - f(A_{j - 1}) - (f(A_j) - f(A_{j - 1})) = 0.
\end{align}
Second, let $i = a_k$. Then:
\begin{align}
  \by_{k - 1}(i) - \by_k(i)
  & = f(A_{k - 1} + A^\ast_i) - f(A_{k - 1} + A^\ast_{i - 1}) - (f(A_k) - f(A_{k - 1})) \nonumber \\
  & = f(A_k + A^\ast_{i - 1}) - f(A_{k - 1} + A^\ast_{i - 1}) - (f(A_k) - f(A_{k - 1})) \nonumber \\
  & \leq 0.
\end{align}
The first equality is due to definition \eqref{eq:augmentation}. The second equality follows from the assumption that $i = a_k$. The inequality is due to the submodularity of $f$. Finally, let $i \notin A_k$. Then:
\begin{align}
  \by_{k - 1}(i) - \by_k(i)
  & = f(A_{k - 1} + A^\ast_i) - f(A_{k - 1} + A^\ast_{i - 1}) -
  (f(A_k + A^\ast_i) - f(A_k + A^\ast_{i - 1})) \nonumber \\
  & \geq 0.
\end{align}
The equality is due to definition \eqref{eq:augmentation}. The inequality is due to the submodularity of $f$.
\end{proof}

\noindent Lemma~\ref{lem:gain difference} says that $\by_{k - 1}(a_k) - \by_k(a_k)$ is the only non-positive entry in $\by_{k - 1} - \by_k$. Since $\by_{k - 1}$ and $\by_k$ are bases, $\displaystyle \sum_{e = 1}^L \by_{k - 1}(i) = \sum_{e = 1}^L \by_k(i) = K$, it follows that $\displaystyle \by_{k - 1}(a_k) - \by_k(a_k) = - \sum_{i \notin A_k} [\by_{k - 1}(i) - \by_k(i)]$. The quantity $\by_{k - 1}(i) - \by_k(i)$ can be viewed as a fraction of item $i$ in $\by_{k - 1}$ exchanged for item $a_k$ in $\by_k$. In the rest of our analysis, we represented these fractions as a vector:
\begin{align}
  \delta(a_k, i) = \max\set{\by_{k - 1}(i) - \by_k(i), 0}.
  \label{eq:exchange}
\end{align}

\begin{lemma}
\label{lem:return difference} For any $k$, the difference in the expected returns of augmentations $\by_{k - 1}$ and $\by_k$ is bounded as:
\begin{align*}
  \langle\bar{\bw}, \by_{k - 1} - \by_k\rangle \leq
  \sum_{e^\ast = 1}^{\rho(a_k)} \Delta_{a_k, e^\ast} \delta(a_k, e^\ast).
\end{align*}
\end{lemma}
\begin{proof}
The claim is proved as:
\begin{align}
  \langle\bar{\bw}, \by_{k - 1} - \by_k\rangle
  & = \sum_{e^\ast \notin A_k} \bar{\bw}(e^\ast) \delta(a_k, e^\ast) -
  \bar{\bw}(a_k) \sum_{e^\ast \notin A_k} \delta(a_k, e^\ast) \nonumber \\
  & = \sum_{e^\ast \notin A_k} \underbrace{(\bar{\bw}(e^\ast) -
  \bar{\bw}(a_k))}_{\Delta_{a_k, e^\ast}} \delta(a_k, e^\ast) \nonumber \\
  & \leq \sum_{e^\ast = 1}^{a_k - 1} \I{e^\ast \notin A_k}
  \Delta_{a_k, e^\ast} \delta(a_k, e^\ast) \nonumber \\
  & = \sum_{e^\ast = 1}^{\rho(a_k)} \Delta_{a_k, e^\ast} \delta(a_k, e^\ast).
\end{align}
The first two steps follow from Lemma~\ref{lem:gain difference} and the subsequent discussion. Then we neglect the negative gaps. Finally, because $f$ is monotonic and submodular, $\delta(a_k, e^\ast) = 0$ for any $e^\ast \notin A_k$ such that $\bx^\ast(e^\ast) = 0$. As a result, we can restrict the scope of the summation over $e^\ast$ to between $1$ and $\rho(a_k)$.
\end{proof}

\noindent Now we are ready to prove our main lemma.

\begin{theorem}
\label{thm:regret decomposition} The expected regret of choosing any basis $\bx$ in episode $t$ is bounded as:
\begin{align*}
  R(\bx, \bar{\bw}) \leq
  \sum_{e = 1}^L \sum_{e^\ast = 1}^{\rho(e)} \Delta_{e, e^\ast} \delta(e, e^\ast),
\end{align*}
where $\delta(e, e^\ast)$ is the fraction of item $e^\ast$ exchanged for item $e$ in episode $t$, and is defined in \eqref{eq:exchange}. Moreover, when $\delta(e, e^\ast) > 0$, $\opm$ observes the weight of item $e$ and $U_t(e) \geq U_t(e^\ast)$. Finally:
\begin{align*}
  \forall t: \sum_{e = 1}^L \sum_{e^\ast = 1}^{\rho(e)} \delta(e, e^\ast) \leq K, \qquad
  \forall t, e \in E: \sum_{e^\ast = 1}^{\rho(e)} \delta(e, e^\ast) \leq 1.
\end{align*}
\end{theorem}
\begin{proof}
The first claim is proved as follows:
\begin{align}
  R(\bx, \bar{\bw}) =
  \langle\bar{\bw}, \bx^\ast - \bx\rangle =
  \sum_{k = 1}^L \langle\bar{\bw}, \by_{k - 1} - \by_k\rangle \leq
  \sum_{k = 1}^L \sum_{e^\ast = 1}^{\rho(a_k)} \Delta_{a_k, e^\ast} \delta(a_k, e^\ast).
\end{align}
First, we rewrite the regret $\langle\bar{\bw}, \bx^\ast - \bx\rangle$ as the sum of the differences in $(L + 1)$ $k$-augmentations, from $\by_0$ to $\by_L$. Note that $\by_0 = \bx^\ast$ and $\by_L = \bx$. Second, we bound each term is the sum using Lemma~\ref{lem:return difference}. Finally, we replace the sum over all indices $k$ by the sum over all items $e$.

The second claim is proved  as follows. Let $\delta(e, e^\ast) > 0$. Then $\opm$ is guaranteed to observe the weight of item $e$ because $\bx(e) \geq \delta(e, e^\ast) > 0$. Furthermore, let $U_t(e) < U_t(e^\ast)$. Then $\opm$ chooses item $e^\ast$ before item $e$, and $\delta(e, e^\ast) = 0$ by Lemma~\ref{lem:gain difference}. This is a contradiction because $\delta(e, e^\ast) > 0$ by our assumption. As a result, it must be true that $\delta(e, e^\ast) > 0$ implies $U_t(e) \geq U_t(e^\ast)$.

The last two inequalities follow from two observations. First, $\sum_{e^\ast = 1}^{\rho(e)} \delta(e, e^\ast) \leq \bx(e)$ for any basis $\bx$ and item $e$, the sum of the contributions from items $e^\ast$ to $e$ cannot be larger than the total contribution of item $e$ in $\bx$. Second, by the definitions in Section~\ref{sec:polymatroids}, $\bx(e) \leq 1$ and $\sum_{e \in E} \bx(e) = K$ for any basis $\bx$ and item $e$.
\end{proof}

\noindent Note that $\delta(e, e^\ast)$ is a random variable that depends on the basis $\bx_t$ in episode $t$. To stress this dependence, we denote it by $\delta_t(e, e^\ast)$ in the rest of our analysis.

\subsection{Upper Bounds}
\label{sec:upper bounds}

Our first result is a gap-dependent bound. We prove a gap-free bound in sequel.

\begin{theorem}[gap-dependent bound]
\label{thm:gap-dependent} The expected cumulative regret of $\opm$ is bounded as:
\begin{align*}
  R(n) \leq
  \sum_{e = 1}^L \frac{16}{\Delta_{e, \rho(e)}} \log n +
  \sum_{e = 1}^L \sum_{e^\ast = 1}^{\rho(e)} \Delta_{e, e^\ast} \frac{4}{3} \pi^2.
\end{align*}
\end{theorem}
\begin{proof}
First, we bound the expected regret in episode $t$ using Theorem~\ref{thm:regret decomposition}:
\begin{align}
  R(n)
  & = \sum_{t = 1}^n \E{\E{R(\bx_t, \bw_t)}{\bw_t}}{\bw_1, \dots, \bw_{t - 1}} \nonumber \\
  & \leq \sum_{t = 1}^n \E{\sum_{e = 1}^L \sum_{e^\ast = 1}^{\rho(e)}
  \Delta_{e, e^\ast} \delta_t(e, e^\ast)}{\bw_1, \dots, \bw_{t - 1}} \nonumber \\
  & = \sum_{e = 1}^L \sum_{e^\ast = 1}^{\rho(e)} \Delta_{e, e^\ast}
  \E{\sum_{t = 1}^n \delta_t(e, e^\ast)}{\bw_1, \dots, \bw_n}.
\end{align}
Second, we bound the regret associated with each item $e$. The key idea is to decompose $\delta_t(e, e^\ast)$ as:
\begin{align}
  \delta_t(e, e^\ast) =
  \delta_t(e, e^\ast) \I{T_{t - 1}(e) \leq \ell_{e, e^\ast}} + \delta_t(e, e^\ast) \I{T_{t - 1}(e) > \ell_{e, e^\ast}}
\end{align}
and then select $\ell_{e, e^\ast}$ appropriately. By Lemma~\ref{lem:pulls} in Appendix, the regret corresponding to $\I{T_{t - 1}(e) > \ell_{e, e^\ast}}$ is bounded as:
\begin{align}
  \sum_{e^\ast = 1}^{\rho(e)} \Delta_{e, e^\ast}
  \E{\sum_{t = 1}^n \delta_t(e, e^\ast) \I{T_{t - 1}(e) > \ell_{e, e^\ast}}}{\bw_1, \dots, \bw_n} \leq
  \sum_{e^\ast = 1}^{\rho(e)} \Delta_{e, e^\ast} \frac{4}{3} \pi^2
\end{align}
when $\ell_{e, e^\ast} = \floors{\frac{8}{\Delta_{e, e^\ast}^2} \log n}$. At the same time, the regret corresponding to $\I{T_{t - 1}(e) \leq \ell_{e, e^\ast}}$ is bounded as:
\begin{align}
  & \sum_{e^\ast = 1}^{\rho(e)} \Delta_{e, e^\ast}
  \E{\sum_{t = 1}^n \delta_t(e, e^\ast) \I{T_{t - 1}(e) \leq \ell_{e, e^\ast}}}
  {\bw_1, \dots, \bw_n} \leq \nonumber \\
  & \qquad \max_{\bw_1, \dots, \bw_n} \left[\sum_{t = 1}^n \sum_{e^\ast = 1}^{\rho(e)}
  \Delta_{e, e^\ast} \delta_t(e, e^\ast) \I{T_{t - 1}(e) \leq \frac{8}{\Delta_{e, e^\ast}^2} \log n}\right].
  \label{eq:trivial regret}
\end{align}
The next step is due to three observations. First, the gaps $\Delta_{e, e^\ast}$ are ordered such that $\Delta_{e, 1} \geq \ldots \geq \Delta_{e, \rho(e)}$. Second, by Theorem~\ref{thm:regret decomposition}, $T_{t - 1}(e)$ increases by one when $\delta_t(e, e^\ast) > 0$, because this event implies that item $e$ is observed. Finally, by Theorem~\ref{thm:regret decomposition}, $\sum_{e^\ast = 1}^{\rho(e)} \delta_t(e, e^\ast) \leq 1$ for all $e$ and $t$. Based on these facts, two of which follow from the structure of a polymatroid, the bound in \eqref{eq:trivial regret} can be bounded from above by:
\begin{align}
  \left[\Delta_{e, 1} \frac{1}{\Delta_{e, 1}^2} + \sum_{e^\ast = 2}^{\rho(e)} \Delta_{e, e^\ast}
  \left(\frac{1}{\Delta_{e, e^\ast}^2} - \frac{1}{\Delta_{e, e^\ast - 1}^2}\right)\right] 8 \log n.
  \label{eq:flush}
\end{align}
By Lemma~\ref{lem:multiple optimal pulls} in Appendix, the above quantity is further bounded by $\frac{16}{\Delta_{e, \rho(e)}} \log n$. Finally, we combine all of our inequalities and get:
\begin{align}
  \sum_{e^\ast = 1}^{\rho(e)} \Delta_{e, e^\ast}
  \E{\sum_{t = 1}^n \delta_t(e, e^\ast)}{\bw_1, \dots, \bw_n} \leq
  \frac{16}{\Delta_{e, \rho(e)}} \log n + \sum_{e^\ast = 1}^{\rho(e)} \Delta_{e, e^\ast} \frac{4}{3} \pi^2.
  \label{eq:per-item regret}
\end{align}
Our main claim is obtained by summing over all items $e$.
\end{proof}

\begin{theorem}[gap-free bound]
\label{thm:gap-free} The expected cumulative regret of $\opm$ is bounded as:
\begin{align*}
  R(n) \leq 8 \sqrt{K L n\log n} + \frac{4}{3} \pi^2 L^2.
\end{align*}
\end{theorem}
\begin{proof}
The main idea is to decompose the expected cumulative regret of $\opm$ into two parts, where the gaps are larger than $\eps$ and at most $\eps$. We analyze each part separately and then select $\eps$ to get the desired result.

Let $\rho_\eps(e)$ be the number of items whose expected weight is higher than that of item $e$ by more than $\eps$ and:
\begin{align}
  Z_{e, e^\ast}(n) = \E{\sum_{t = 1}^n \delta_t(e, e^\ast)}{\bw_1, \dots, \bw_n}.
\end{align}
Then for any $\eps$, the regret of $\opm$ can be decomposed as:
\begin{align}
  R(n) =
  \sum_{e = 1}^L \sum_{e^\ast = 1}^{\rho_\eps(e)} \Delta_{e, e^\ast} Z_{e, e^\ast}(n) +
  \sum_{e = 1}^L \sum_{e^\ast = \rho_\eps(e) + 1}^{\rho(e)} \Delta_{e, e^\ast} Z_{e, e^\ast}(n).
  \label{eq:gap-free split}
\end{align}
The first term can be bounded similarly to \eqref{eq:per-item regret}:
\begin{align}
  \sum_{e = 1}^L \sum_{e^\ast = 1}^{\rho_\eps(e)} \Delta_{e, e^\ast} Z_{e, e^\ast}(n)
  & \leq \sum_{e = 1}^L \frac{16}{\Delta_{e, \rho_\eps(e)}} \log n +
  \sum_{e = 1}^L \sum_{e^\ast = \rho_\eps(e) + 1}^{\rho(e)} \Delta_{e, e^\ast} \frac{4}{3} \pi^2
  \nonumber \\
  & \leq \frac{16}{\eps} L \log n + \frac{4}{3} \pi^2 L^2.
\end{align}
The second term is bounded trivially as:
\begin{align}
  \sum_{e = 1}^L \sum_{e^\ast = \rho_\eps(e) + 1}^{\rho(e)} \Delta_{e, e^\ast} Z_{e, e^\ast}(n) \leq
  \eps K n
\end{align}
because $\displaystyle \sum_{e = 1}^L \sum_{e^\ast = 1}^{\rho(e)} \delta_t(e, e^\ast) \leq K$ in any episode $t$ (Theorem~\ref{thm:regret decomposition}) and $\Delta_{e, e^\ast} \leq \eps$. Finally, we get:
\begin{align}
  R(n) \leq \frac{16}{\eps} L \log n + \eps K n + \frac{4}{3} \pi^2 L^2
\end{align}
and choose $\displaystyle \eps = 4 \sqrt{\frac{L \log n}{K n}}$. This concludes our proof.
\end{proof}

\subsection{Improvement over General Upper Bounds}
\label{sec:upper bound improvement}

$\opm$ is an instance of the UCB algorithm by \cite{gai12combinatorial} for combinatorial semi-bandits (Section~\ref{sec:related work}). So it is natural to ask if our upper bounds on the regret of $\opm$ (Section~\ref{sec:upper bounds}) are tighter than those in stochastic combinatorial semi-bandits \citep{chen13combinatorial,kveton14tight}. In this section, we show that this is the case, by comparing our upper bounds to \cite{kveton14tight}.

Our $O(\sqrt{K L n \log n})$ gap-free upper bound (Theorem~\ref{thm:gap-free}) has the same dependence on $K$, $L$, and $n$ as the upper bound of \cite{kveton14tight} (Theorem 6). The only notable improvement in our analysis is that we reduce the constant at the $\sqrt{n \log n}$ term from $47$ to $8$.

Our $O(L (1 / \Delta) \log n)$ upper bound (Theorem~\ref{thm:gap-dependent}) is tighter by a factor of $K$ than the $O(K L (1 / \Delta) \log n)$ upper bound of \cite{kveton14tight} (Theorem 5). However, our notion of the gap, item-based \eqref{eq:gap}, differs from that of \cite{kveton14tight}, solution-based. So hypothetically, the improvement in our bound may be solely due to a different notion of the gap. In the rest of this section, we argue that this is not the case.

Specifically, we consider the following \emph{uniform matroid bandit}. Let $E = \set{1, \dots, L}$ be a set of $L$ items and the family of independent sets be defined as:
\begin{align}
  \cI = \set{I \subseteq E: \abs{I} \leq K},
  \label{eq:uniform matroid independence}
\end{align}
which means that any set of up to $K$ items is feasible. Then $M = (E, \cI)$ is a  rank-$K$ \emph{uniform matroid}. Let $P$ be a distribution over the weights of the items, where the weight of each item is distributed independently of the other items. The weight of item $e$ is drawn i.i.d. from a Bernoulli distribution with mean:
\begin{align}
  \bar{\bw}(e) = \left\{
  \begin{array}{ll}
  0.5 & e \leq K\\
  0.5 - \Delta & \textrm{otherwise},
  \end{array}
  \right.
  \label{eq:uniform matroid weights}
\end{align}
where $0 < \Delta < 0.5$. Then $B_\text{unif}  = (M, P)$ is our uniform matroid bandit. The optimal solution to $B_\text{unif}$ is $A^\ast = \set{1, \dots, K}$, the first $K$ items with the largest weights.

The key property of $B_\text{unif}$ is that our gaps coincide with those of \cite{kveton14tight}. In particular, for any suboptimal item $e \notin A^\ast$, the difference between the returns of $A^\ast$ and the best suboptimal solution that contains item $e$ is $\Delta$; the same as the difference between the returns of item $e$ and any optimal item $e^\ast \in A^\ast$ \eqref{eq:gap}. Because the gaps $\Delta$ are the same, our $O(L (1 / \Delta) \log n)$ bound is indeed a factor of $K$ tighter than the $O(K L (1 / \Delta) \log n)$ bound of \cite{kveton14tight}.

\subsection{Lower Bounds}
\label{sec:lower bounds}

We prove gap-dependent and gap-free lower bounds on the regret in polymatroid bandits. These bounds are derived on a class of polymatroid bandits that are equivalent to $K$ independent Bernoulli bandits.

Specifically, we consider the following \emph{partition matroid bandit}. Let $E = \set{1, \dots, L}$ be a set of $L$ items and $B_1, \dots, B_K$ be a partition of this set such that $\abs{B_i} = L / K$, where $L / K$ is an integer. Let the family of independent sets be defined as:
\begin{align}
  \cI = \set{I \subseteq E: \left(\forall k: \abs{\cI \cap B_k} \leq 1\right)}.
  \label{eq:partition matroid independence}
\end{align}
Then $M = (E, \cI)$ is a \emph{partition matroid} of rank $K$. Let $P$ be a probability distribution over the weights of the items, where the weight of each item is distributed independently of the other items. The weight of item $e$ is drawn i.i.d. from a Bernoulli distribution with mean:
\begin{align}
  \bar{\bw}(e) = \left\{
  \begin{array}{ll}
  0.5 & \exists k: e = \min_{i \in B_k} i \\
  0.5 - \Delta & \textrm{otherwise},
  \end{array}
  \right.
  \label{eq:partition matroid weights}
\end{align}
where $0 < \Delta < 0.5$. Then $B_\text{part} = (M, P)$ is our partition matroid bandit. The key property of $B_\text{part}$ is that it is equivalent to $K$ independent Bernoulli bandits with $L / K$ arms each. The optimal item in each bandit is the item with the smallest index. So the optimal solution is $A^\ast = \set{e: (\exists k: e = \min_{i \in B_k} i)}$. We also note that all gaps \eqref{eq:gap} are $\Delta$.

To formalize our gap-dependent lower bound, we introduce the notion of \emph{consistent algorithms}. We say that the algorithm is consistent if for any partition matroid bandit, any $e \notin A^\ast$, and any $\alpha > 0$, $\mathbb{E}[T_n(e)] = o(n^\alpha)$, where $T_n(e)$ is the number of times that item $e$ is observed in $n$ episodes. In the rest of our analysis, we focus only on consistent algorithms. This is without loss of generality. In particular, by the definition of consistency, inconsistent algorithms perform poorly on some instances of our problem and therefore cannot achieve logarithmic regret on all instances.

\begin{proposition}
\label{prop:gap-dependent lower bound} For any $L$ and $K$ such that $L / K$ is an integer, and any $\Delta$ such that $0< \Delta < 0.5$, the regret of any consistent algorithm on partition matroid bandit $B_\text{part}$ is bounded from below as:
\begin{align*}
  \liminf_{n \to \infty} \frac{R(n)}{\log n} \geq \frac{L - K}{4 \Delta}.
\end{align*}
\end{proposition}
\begin{proof}
The theorem is proved as follows:
\begin{align}
  \liminf_{n \to \infty} \frac{R(n)}{\log n}
  & \geq \sum_{k = 1}^K \sum_{e \in B_k - A^\ast}
  \frac{\Delta}{\mathrm{KL}(0.5 - \Delta, 0.5)} \nonumber \\
  & = \frac{(L - K) \Delta}{\mathrm{KL}(0.5 - \Delta, 0.5)} \nonumber \\
  & \geq \frac{L - K}{4 \Delta},
\end{align}
where $\mathrm{KL}(0.5 - \Delta, 0.5)$ is the KL divergence between two Bernoulli variables with the means of $0.5 - \Delta$ and $0.5$. The first inequality is due to an existing lower bound for Bernoulli bandits \citep{lai85asymptotically}, which is applied separately to each part $B_k$. The second inequality follows from $\mathrm{KL}(p, q) \leq \frac{(p - q)^2}{q (1 - q)}$, where $p = 0.5 - \Delta$ and $q = 0.5$.
\end{proof}

\noindent Now we prove a gap-free lower bound.

\begin{proposition}
\label{prop:gap-free lower bound} For any $L$ and $K$ such that $L / K$ is an integer, and any $n > 0$, the regret of any algorithm on partition matroid bandit $B_\text{part}$ is bounded from below as:
\begin{align*}
  R(n) \geq \frac{1}{20} \min(\sqrt{K L n}, K n).
\end{align*}
\end{proposition}
\begin{proof}
The matroid bandit $B_\text{part}$ can be viewed as $K$ independent Bernoulli bandits with $L / K$ arms each. By Theorem 5.1 of \cite{AuCBFSch02}, for any time horizon $n$, the gap $\Delta$ can be chosen such that the regret of any algorithm on any of the $K$ bandits is at least $\frac{1}{20} \min(\sqrt{(L / K) n}, n)$. So the regret due to all bandits is at least:
\begin{align}
  K \frac{1}{20} \min\set{\sqrt{(L / K) n}, n} =
  \frac{1}{20} \min\set{\sqrt{K L n}, K n}.
\end{align}
Note that the bound of \cite{AuCBFSch02} is stated for the adversarial setting. However, because the worst-case environment in the proof is stochastic, it also applies to our problem.
\end{proof}

\subsection{Discussion of Theoretical Results}
\label{sec:discussion}

We prove two upper bounds on the expected cumulative regret of $\opm$:
\begin{align*}
  \text{Theorem~\ref{thm:gap-dependent}}: & \qquad O(L (1 / \Delta) \log n), \\
  \text{Theorem~\ref{thm:gap-free}}: & \qquad O(\sqrt{K L n \log n}),
\end{align*}
where the gap is $\Delta = \min\limits_e \min\limits_{e^\ast \leq \rho(e)} \Delta_{e, e^\ast}$. Both bounds are at most linear in $K$ and $L$, and sublinear in $n$. In other words, they scale favorably with all quantities of interest and therefore we expect them to be practical. Our $O(L (1 / \Delta) \log n)$ upper bound matches the lower bound in Proposition~\ref{prop:gap-dependent lower bound} up to a constant and therefore is tight. It is also a factor of $K$ tighter than the $O(K L (1 / \Delta) \log n)$ upper bound of \cite{kveton14tight} for a more general class of problems, stochastic combinatorial semi-bandits (Section~\ref{sec:upper bound improvement}). Our $O(\sqrt{K L n \log n})$ upper bound matches the lower bound in Proposition~\ref{prop:gap-free lower bound} up to a factor of $\sqrt{\log n}$.

Our gap-dependent upper bound has the same form as the bound of \citet{auer02finitetime} for multi-armed bandits. This suggests that the sample complexity of learning the maximum-weight basis of a polymatroid is similar to that of the multi-armed bandit problem. The only major difference is in the definitions of the gaps. In other words, learning in polymatroids is extremely sample efficient.

The key step in our analysis is showing that the difference in the expected returns of bases $\bx^\ast$ and $\bx_t$ can be expressed as the sum of the differences in the expected returns of intermediate solutions, all of which are bases such that the difference in the gains of any two consecutive bases has at most one negative entry. This decomposition is highly non-trivial and is derived based on the submodularity of our problem. An important aspect of our analysis is that the terms $\delta_t(e, e^\ast)$ (Theorem~\ref{thm:regret decomposition}) are not bounded until it is necessary, such as in \eqref{eq:flush}. Therefore, our upper bounds are tight and do not contain quantities that are not native to our problem, such as the maximum number of non-zero entries in $\bx \in \Theta$. In fact, under the assumption that $f(e) \leq 1$ for all items $e \in E$ (Section~\ref{sec:polymatroids}), our notion of complexity $K = f(E)$ is never larger than the maximum number of non-zero entries in any feasible solution $\bx \in \Theta$, a common quantity in the regret bounds of combinatorial bandits \citep{gai12combinatorial,kveton14tight,cesabianchi12combinatorial,audibert14regret}.


\section{Experiments}
\label{sec:experiments}

We conduct three experiments. In Section~\ref{sec:min-cost flow}, we evaluate the tightness of our regret bounds on a synthetic problem. In Section~\ref{sec:MST}, we apply $\opm$ to the problem of learning routing networks. Finally, in Section~\ref{sec:diverse recommendations}, we evaluate $\opm$ on the problem of recommending diverse movies.

All experiments are episodic. In each episode, $\opm$ chooses a basis $\bx_t$, observes the weights of all items that contribute to $\bx_t$, and updates its model of the world. In Sections~\ref{sec:MST} and \ref{sec:diverse recommendations}, the performance of $\opm$ is measured by the \emph{expected per-step return} in $n$ episodes:
\begin{align}
  \frac{1}{n} \E{\sum_{t = 1}^n \langle\bw_t, \bx_t\rangle}{\bw_1, \dots, \bw_n},
  \label{eq:per-step return}
\end{align}
which is the expected cumulative return in $n$ episodes divided by $n$. We choose this metric because we want to report the quality of solutions and not just their regret, the difference from the optimal solution.

We compare $\opm$ to two baselines. The first baseline is the maximum-weight basis $\bx^\ast$ \eqref{eq:optimal arm}. This is our notion of optimality. The second baseline is an $\eps$-greedy policy. The policy is implemented similarly to $\opm$. In particular, it is Algorithm~\ref{alg:bandit} that is modified as follows. In each episode, $U_t(e)$ is set to $\hat{\bw}_{T_{t - 1}(e)}(e)$ for all items $e$ with probability $1 - \eps$. With probability $\eps$, $U_t(e)$ is chosen randomly for all items $e$. The exploration rate is set as $\eps = 0.1$. In all of our experiments, this is the best performing $\eps$-greedy policy from the class of $\eps$-greedy policies where $\eps \in \set{0, 0.1, \dots, 1}$.


\subsection{Minimum-Cost Flow}
\label{sec:min-cost flow}

In the first experiment, we evaluate $\opm$ on a synthetic problem of learning minimum-cost flows. The experiment shows that our $O(L (1 / \Delta) \log n)$ gap-dependent upper bound is practical. We experiment with larger values of $\Delta$. In this setting, our gap-dependent upper bound is tighter than the gap-free one.

\begin{figure*}[t]
  \centering
  \includegraphics[width=6.4in, bb=1.25in 4in 9.25in 6.5in]{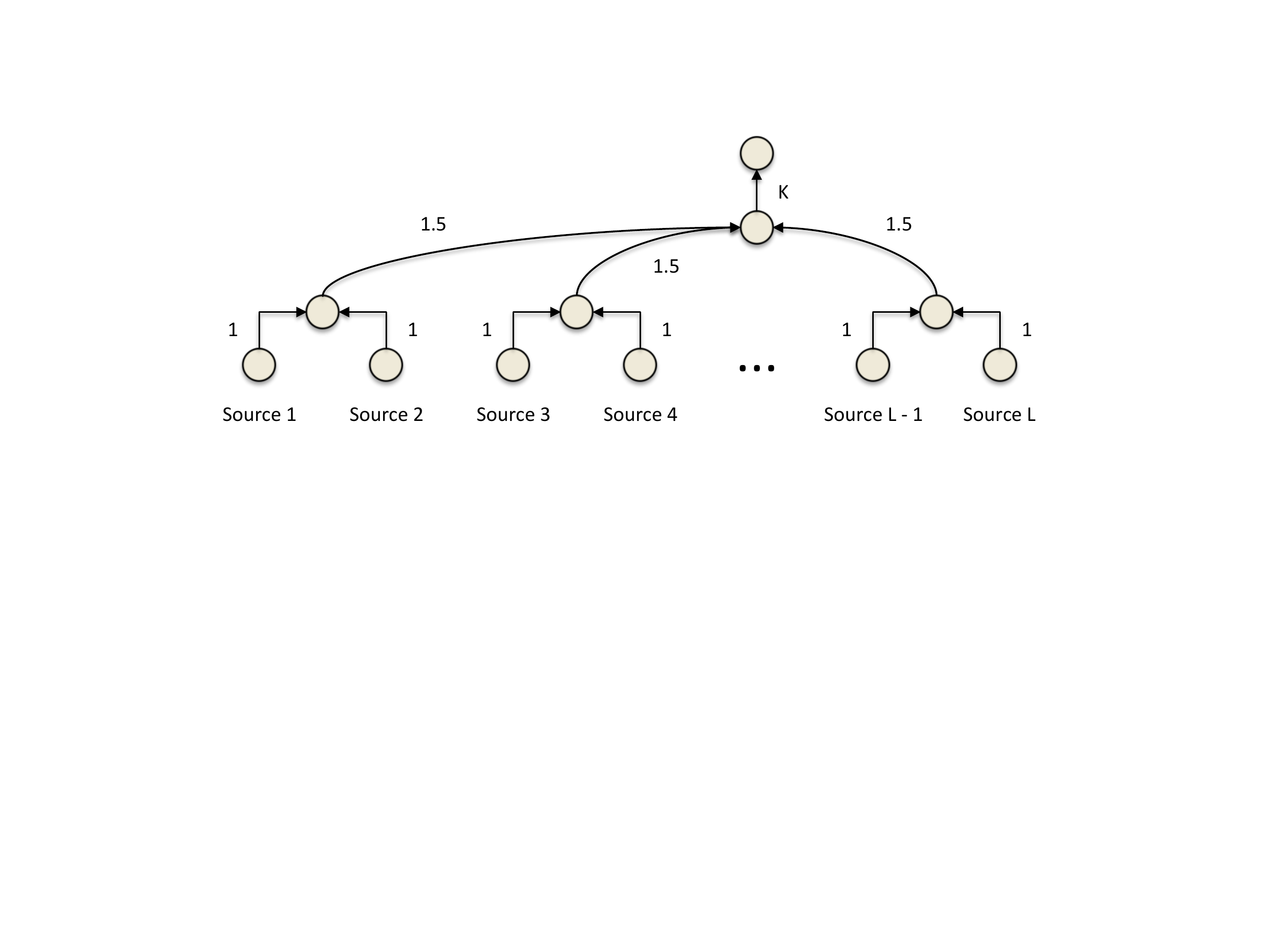}
  \caption{The flow network in Section~\ref{sec:min-cost flow}. The network contains $L$ source nodes and the maximum flow is $K$. The capacity of the link is shown next to the link.}
  \label{fig:flow network}
\end{figure*}

We experiment with a flow network with $L$ source nodes and one sink node. The network is illustrated in Figure~\ref{fig:flow network}. The network is defined by three constraints. First, the maximum flow through any source node is $1$. Second, the maximum flow through any two consecutive source nodes, $e$ and $e + 1$ where $e = 2 i - 1$ for $i \in \set{1, \dots, L / 2}$, is $\frac{3}{2}$. Third, the maximum flow is $K$. We assume that $K$ is an integer multiple of $\frac{3}{2}$. The cost of the flow from source node $e$ is a Bernoulli random variable with mean:
\begin{align}
  \bar{\bw}(e) = \left\{
  \begin{array}{ll}
  0.5 - \Delta / 2 & e \leq \frac{4}{3} K \\
  0.5 + \Delta / 2 & \text{otherwise}.
  \end{array}
  \right.
  \label{eq:flow weight}
\end{align}
Our problem is parametrized by $K$, $L$, and $\Delta$. The optimal solution to the problem is to pass the maximum flow through the first $\frac{4}{3} K$ source nodes.

Our problem can be formulated as minimizing a modular function on a polymatroid. The ground set $E$ are $L$ source nodes. The submodular function $f$ captures the structure of the network and is defined as:
\begin{align}
  f(X) = \min\set{\sum_{i = 1}^{L / 2} \min\set{\I{(2 i - 1) \in X} + \I{2 i \in X}, \frac{3}{2}}, K}.
  \label{eq:flow submodular}
\end{align}
Note that $f(X)$ can be computed in $O(L)$ time, by summing up $L$ indicators. The weight of item $e$ is drawn i.i.d. from a Bernoulli distribution with mean $\bar{\bw}(e)$ in \eqref{eq:flow weight}, independently of the other items.

\begin{figure*}[t]
  \centering
  \includegraphics[width=6.4in, bb=0.5in 4.25in 8.5in 6.75in]{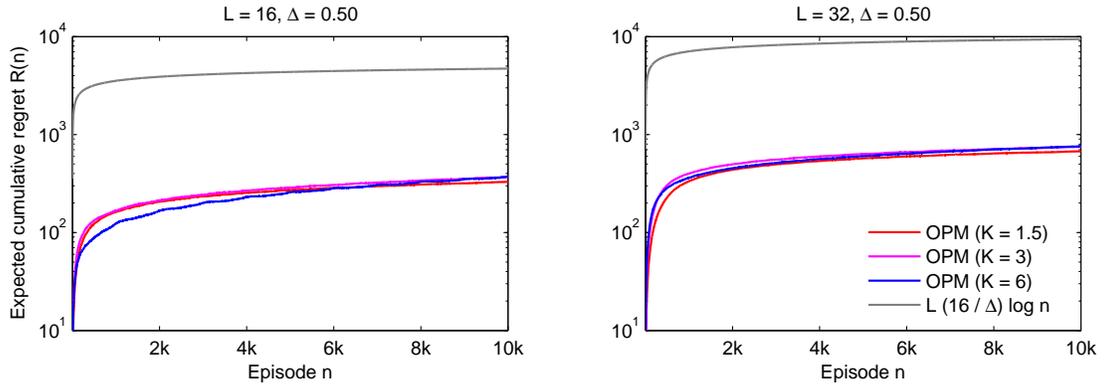}
  \caption{The regret of $\opm$ as a function of the number of episodes $n$. The regret is averaged over $100$ runs.}
  \label{fig:regret time}
\end{figure*}

In Figure~\ref{fig:regret time}, we report the regret of $\opm$ as a function of the number of episodes $n$ for various settings of $K$ and $L$. The gap is $\Delta = 0.5$. We observe three major trends. First, the regret grows on the order of $\log n$, as suggested by our $O(L (1 / \Delta) \log n)$ upper bound. Second, the regret does not change much with $K$. This is consistent with the fact that our bound is independent of $K$. Finally, we note that the bound is surprisingly tight. In particular, for larger values of $n$, it is only about $10$ times larger than the actual regret.

\begin{table*}[t]
  \centering
  \begin{tabular}{rr|rrr|rrr} \hline
    $L$ & $K$ &
    $\Delta$ & Regret $R(n)$ & $L (16 / \Delta) \log n$ &
    $\Delta$ & Regret $R(n)$ & $L (16 / \Delta) \log n$ \\ \hline
    16 & 1.50 &
    0.5 & $329.1 \pm 2.5$ & 4,716 &
    0.25 & $577.6 \pm 4.1$ & 9,431 \\
    16 & 3.00 &
    0.5 & $368.6 \pm 3.4$ & 4,716 &
    0.25 & $599.7 \pm 4.3$ & 9,431 \\
    16 & 6.00 &
    0.5 & $373.3 \pm 4.8$ & 4,716 &
    0.25 & $546.1 \pm 5.6$ & 9,431 \\
    32 & 1.50 &
    0.5 & $675.5 \pm 3.0$ & 9,431 &
    0.25 & $1,182.6 \pm 6.2$ & 18,863 \\
    32 & 3.00 &
    0.5 & $748.8 \pm 3.9$ & 9,431 &
    0.25 & $1,356.1 \pm 6.0$ & 18,863 \\
    32 & 6.00 &
    0.5 & $759.6 \pm 4.3$ & 9,431 &
    0.25 & $1,299.3 \pm 6.9$ & 18,863 \\ \hline
  \end{tabular}
  \caption{The regret of $\opm$ as a function of $K$, $L$, and $\Delta$. The regret is averaged over $100$ runs.}
  \label{tab:regret dependency}
\end{table*}

In Table~\ref{tab:regret dependency}, we report the regret of $\opm$ in $10$k episodes for various settings of $K$, $L$, and $\Delta$. We observe that the regret depends on $L$ and $\Delta$ as suggested by our $O(L (1 / \Delta) \log n)$ upper bound. In particular, it does not change much with $K$, and it doubles as we double $L$ or halve $\Delta$. We note again that our upper bound is surprisingly tight, never more than $20$ times larger than the actual regret.


\subsection{Minimum Spanning Tree}
\label{sec:MST}

\begin{figure*}[t]
  \centering
  \includegraphics[width=6.4in, bb=0.5in 4.25in 8.5in 6.75in]{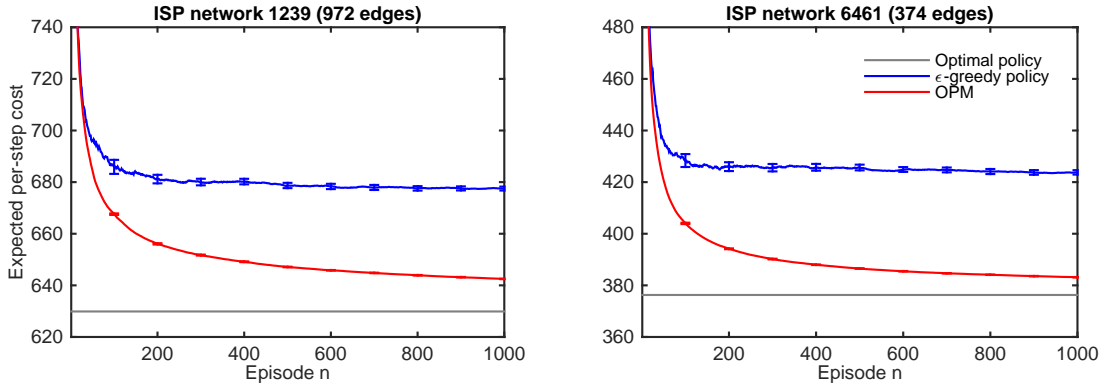}
  \caption{The expected per-step cost of building two minimum spanning trees in up to $10^3$ episodes.}
  \label{fig:latency trends}
\end{figure*}

\begin{table*}[t]
  \centering
  {\small
  \begin{tabular}{l|cc|ccc|rrr} \hline
    ISP & Number & Number & Minimum & Maximum & Average & Optimal & $\eps$-greedy & \\
    network & of nodes & of edges & latency & latency & latency & policy & policy & $\opm$ \\ \hline
    1221 & 108 & 153 & 1 & 17 & 2.78 & 305.00 & $307.26 \pm 0.06$ & $305.82 \pm 0.07$ \\
    1239 & 315 & 972 & 1 & 64 & 3.20 & 629.88 & $677.51 \pm 0.77$ & $642.49 \pm 0.09$ \\
    1755 & 87 & 161 & 1 & 31 & 2.91 & 192.81 & $199.56 \pm 0.13$ & $194.92 \pm 0.07$ \\
    3257 & 161 & 328 & 1 & 47 & 4.30 & 550.85 & $570.21 \pm 0.40$ & $560.11 \pm 0.08$ \\
    3967 & 79 & 147 & 1 & 44 & 5.19 & 306.80 & $321.21 \pm 0.22$ & $308.66 \pm 0.05$ \\
    6461 & 141 & 374 & 1 & 45 & 6.32 & 376.27 & $423.78 \pm 0.85$ & $383.15 \pm 0.08$ \\ \hline
  \end{tabular}
  }
  \caption{The description of $6$ ISP networks from our experiments and the expected per-step costs of building minimum spanning trees on these networks in $10^3$ episodes. All latencies and costs are reported in milliseconds.}
  \label{tab:latency summary}
\end{table*}

In the second experiment, we apply $\opm$ to the problem of learning routing networks for an Internet service provider (ISP). The routing network is a spanning tree \citep{oliveira05survey}. Our goal is to identify the spanning tree that has the lowest expected latency on its edges. Note that this is a minimization problem. Therefore, we refer to the return of a policy as its \emph{cost}.

Our problem can be formulated as a \emph{graphic matroid bandit} \citep{kveton14matroid}, which is a form of a polymatroid bandit. The ground set $E$ are the edges of the graph that represents the topology of the network. We experiment with $6$ networks from the \emph{RocketFuel} dataset \citep{spring04measuring}, with up to $300$ nodes and $10^3$ edges (Table~\ref{tab:latency summary}). A set of edges is \emph{independent} if it forms a forest. The corresponding rank function $f$ is defined as:
\begin{align}
  f(X) = \abs{\text{largest subset of $X$ that is a forest}}.
\end{align}
The value of $f(X)$ can be computed naively by a greedy algorithm in $O(\abs{X}^2)$ time. The latency of edge $e$ in episode $t$ is:
\begin{align}
  \bw_t(e) = \bar{\bw}(e) - 1 + \eps,
\end{align}
where $\bar{\bw}(e)$ is the expected latency, which is recorded in our dataset; and $\eps \sim \mathrm{Exp}(1)$ is exponential noise. The latency $\bar{\bw}(e)$ ranges from $1$ to $64$ milliseconds. Our noise model is motivated by the observation that the latency in ISP networks can be mostly explained by geographical distances \citep{choi04analysis}, the expected latency $\bar{\bw}(e)$. The noise tends to be small, on the order of a few hundred microseconds, and it is unlikely to cause high latency.

In Figure~\ref{fig:latency trends}, we report our results on two largest ISP networks. We observe two trends. First, the cost of $\opm$ approaches that of the optimal solution $\bx^\ast$ as the number of episodes increases. Second, $\opm$ performs better than the $\eps$-greedy policy in less than $10$ episodes. The costs of all policies on all networks are reported in Table~\ref{tab:latency summary}. We observe that $\opm$ outperforms the $\eps$-greedy policy, typically by a large margin.

$\opm$ learns quickly because our networks are sparse. In particular, the number of edges in each network is never more than four times larger than the number of edges in its spanning tree. So theoretically, each edge can be observed at least once in four episodes and $\opm$ learns quickly the mean latency of each edge.


\subsection{Diverse recommendations}
\label{sec:diverse recommendations}

\begin{figure*}[t]
  \centering
  \hspace{-0.05in}
  \raisebox{-0.99in}{\includegraphics[width=2.8in, bb=2.5in 4.25in 6in 6.75in]{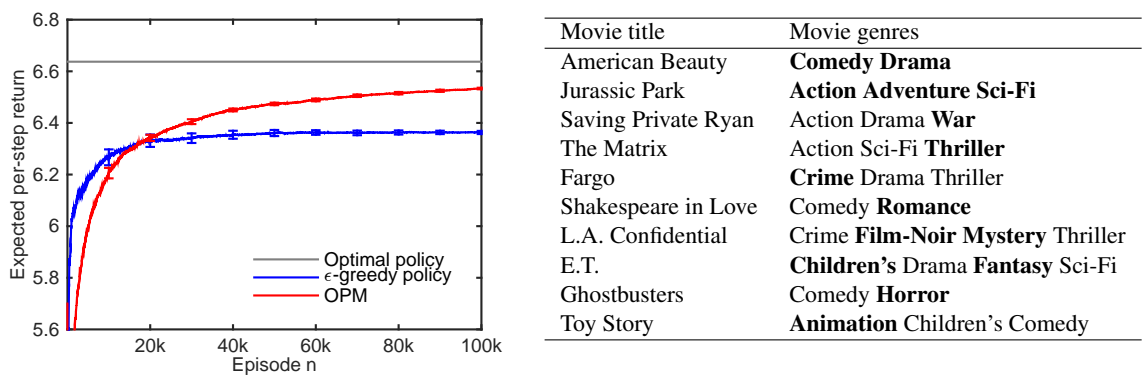}}
  \
  {\small
  \begin{tabular}{ll} \hline
    Movie title & Movie genres \\ \hline
    American Beauty & {\bf Comedy} {\bf Drama} \\
    Jurassic Park & {\bf Action} {\bf Adventure} {\bf Sci-Fi} \\
    Saving Private Ryan & Action Drama {\bf War} \\
    The Matrix & Action Sci-Fi {\bf Thriller} \\
    Fargo & {\bf Crime} Drama Thriller \\
    Shakespeare in Love & Comedy {\bf Romance} \\
    L.A. Confidential & Crime {\bf Film-Noir} {\bf Mystery} Thriller \\
    E.T. & {\bf Children's} Drama {\bf Fantasy} Sci-Fi \\
    Ghostbusters & Comedy {\bf Horror} \\
    Toy Story & {\bf Animation} Children's Comedy \\ \hline
  \end{tabular}
  }
  \caption{\textbf{Left}. The return of three movie recommendation policies in up to $10^5$ episodes. \textbf{Right}. Ten most popular movies in the optimal solution, $\set{e: \bx^\ast(e) > 0}$. These movies are shown in the order of decreasing popularity. The movie genre is highlighted if the associated movie is the most popular movie in that genre.}
  \label{fig:recommendation trends}
\end{figure*}

In the third experiment, $\opm$ is evaluated as a movie recommender. The recommender is used repeatedly by simulated users and the goal is to learn how to recommended diverse movies that maximize the satisfaction of an average user (Section~\ref{sec:optimization}). A system like this could be used in practice to identify trending movies.

We experiment with the \emph{MovieLens} dataset \citep{movielens}, a dataset of $6$ thousand people who assigned one million ratings to $4$ thousand movies. The ground set $E$ are $100$ movies, $50$ most and least rated movies in our dataset. The movies cover $18$ genres. The submodular function $f$ is defined as:
\begin{align}
  f(X) = \abs{\text{movie genres covered by movies $X$}}.
\end{align}
The value of $f(X)$ can be computed in $O(\abs{X})$ time. The weight of movie $e$ in episode $t$ is:
\begin{align}
  \bw_t(e) = \I{\text{user in episode $t$ watches movie $e$}}.
\end{align}
The user in episode $t$ is chosen randomly from our pool of $6$k users. We assume that the user watches movie $e$ if that movie is rated by that user in our dataset. The expected weight $\bar{\bw}(e)$ is the probability that movie $e$ is watched by a randomly chosen user.

Our results are reported in Figure~\ref{fig:recommendation trends}. As in Section~\ref{sec:MST}, we observe two major trends. First, the return of $\opm$ approaches that of the optimal solution $\bx^\ast$ as the number of episodes increases. Second, $\opm$ performs better than the $\eps$-greedy policy after about $20$k episodes. The optimal solution $\bx^\ast$ is visualized in Figure~\ref{fig:recommendation trends}. It contains $13$ movies and therefore is extremely sparse.


\section{Related Work}
\label{sec:related work}

Polymatroids \citep{edmonds70submodular} are a generalization of matroids \citep{whitney35abstract}. Therefore, our work can be viewed as a generalization of matroid semi-bandits \citep{kveton14matroid}. We significantly extend the work of \cite{kveton14matroid} and essentially show that the problem of maximizing a modular function subject to a submodular constraint can be learned efficiently. Our generalization is by far non-trivial. For instance, the key part of our analysis is a novel regret decomposition (Section~\ref{sec:regret decomposition}), which leverages the submodularity of our constraint. This structure is not apparent in the work of \cite{kveton14matroid}.

Our problem is an instance of a stochastic combinatorial semi-bandit \citep{gai12combinatorial}. \cite{gai12combinatorial} proposed and analyzed a UCB-like algorithm for solving this problem. \cite{chen13combinatorial} and \cite{kveton14tight} proved $O(K^2 L (1 / \Delta) \log n)$ and $O(K L (1 / \Delta) \log n)$ upper bounds on the regret of this algorithm, respectively. The latter is tight \citep{kveton14tight}. $\opm$ is an instance of the UCB-like algorithm where the combinatorial optimization oracle is greedy. Our optimization problem is on a polymatroid and therefore we can derive a factor of $K$ tighter gap-dependent regret bound (Theorem~\ref{thm:gap-dependent}) than \cite{kveton14tight}. We note that our gap-free regret bound is of the same magnitude.

COMBAND \citep{cesabianchi12combinatorial}, follow-the-perturbed-leader (FPL) with geometric resampling \citep{neu13efficient}, and online stochastic mirror descent (OSMD) \citep{audibert14regret} are three recently proposed algorithms for adversarial combinatorial semi-bandits. FPL does not achieve the optimal regret, but it is computationally efficient when the offline variant of the combinatorial optimization problem can be solved efficiently \citep{audibert14regret}. OSMD achieves the optimal regret, but it is not guaranteed to be computationally efficient if the projection on the convex hull of the feasible set cannot be implemented efficiently. In our problem, the convex hull is $B_M$ \eqref{eq:base polyhedron} and the projection can be implemented in $O(L^6)$ time \citep{suehiro12online}. So the time complexity of a single step of OSMD is $O(L^6)$. This is several orders of magnitude higher than the time complexity of $\opm$, $O(L \log L)$; and not very practical for large values of $L$. Finally, COMBAND is not guaranteed to be computationally efficient. Based on Section 5.4 of \cite{cesabianchi12combinatorial}, even on the problem of learning the minimum spanning tree, an instance of maximizing a modular function on a polymatroid.

Several recent papers studied the problem of learning how to maximize a submodular function \citep{guillory11online,yue11linear,gabillon13adaptive,wen13sequential,gabillon14largescale}. These papers are only loosely related to our work because they study a different problem, which is learning how to maximize an \emph{unknown submodular function} subject to a cardinality constraint. Our learning problem is maximizing an \emph{unknown modular function} subject to a \emph{known submodular constraint}.


\section{Conclusions}
\label{sec:conclusions}

In this work, we study the problem of learning to act greedily. We formulate the problem as learning how to maximize an unknown modular function on a known polymatroid in the bandit setting. Our formulation is quite general and includes many popular problems, such as learning variants of the minimum spanning tree and minimum-cost flow. We propose a computationally-efficient method for solving the problem and prove two upper bounds on its regret. Our $O(L (1 / \Delta) \log n)$ gap-dependent upper bound is tight up to a constant and our $O(\sqrt{K L n \log n})$ gap-free upper bound is tight up to a factor of $\sqrt{\log n}$. We evaluate our method on three problems, and show that it can learn near-optimal policies computationally and sample efficiently.

We leave open several questions of interest. For instance, our $O(\sqrt{K L n \log n})$ upper bound matches the $\Omega(\sqrt{K L n})$ lower bound only up to a factor of $\sqrt{\log n}$. We strongly believe that this factor can be eliminated by modifying the confidence radius in \eqref{eq:confidence radius} as in \cite{audibert09minimax}. We leave this for future work.

Thompson sampling \citep{thompson33likelihood} often performs better in practice than ${\tt UCB1}$ \citep{auer02finitetime}. We believe that it is relatively straightforward to propose a Thompson-sampling variant of $\opm$, by replacing the UCBs in Algorithm~\ref{alg:bandit} with sampling from the posterior on the mean weights \citep{WenAEK14}. We also believe that the regret of this algorithm is bounded and this can be proved. The reason is that the frequentist analysis of Thompson sampling \citep{agrawal12analysis} resembles that of ${\tt UCB1}$ \citep{auer02finitetime}. As a result, it is likely that the analysis of Thompson-sampling $\opm$ can be carried out similarly to this paper.

In this work, we study one particular problem, maximization of a modular function on a polymatroid, in one particular learning setting, stochastic semi-bandits. It is an open question whether the ideas in our paper generalize to other polymatroid problems, such as maximizing a modular function on the intersection of two matroids \citep{papadimitriou98combinatorial}; and other learning variants of our problem, such as learning in the adversarial setting \citep{AuCBFSch02} or with the full-bandit feedback \citep{dani08stochastic}.

\bibliography{References}

\begin{thebibliography}{33}
\providecommand{\natexlab}[1]{#1}
\providecommand{\url}[1]{\texttt{#1}}
\expandafter\ifx\csname urlstyle\endcsname\relax
  \providecommand{\doi}[1]{doi: #1}\else
  \providecommand{\doi}{doi: \begingroup \urlstyle{rm}\Url}\fi

\bibitem[Agrawal and Goyal(2012)]{agrawal12analysis}
Shipra Agrawal and Navin Goyal.
\newblock Analysis of {Thompson} sampling for the multi-armed bandit problem.
\newblock In \emph{Proceeding of the 25th Annual Conference on Learning
  Theory}, pages 39.1--39.26, 2012.

\bibitem[Ashkan et~al.(2014{\natexlab{a}})Ashkan, Kveton, Berkovsky, and
  Wen]{ashkan14diversified}
Azin Ashkan, Branislav Kveton, Shlomo Berkovsky, and Zheng Wen.
\newblock Diversified utility maximization for recommendations.
\newblock In \emph{Poster Proceedings of the 8th ACM Conference on Recommender
  Systems}, 2014{\natexlab{a}}.

\bibitem[Ashkan et~al.(2014{\natexlab{b}})Ashkan, Kveton, Berkovsky, and
  Wen]{ashkan14dum}
Azin Ashkan, Branislav Kveton, Shlomo Berkovsky, and Zheng Wen.
\newblock {DUM}: Diversity-weighted utility maximization for recommendations.
\newblock \emph{CoRR}, abs/1410.0949, 2014{\natexlab{b}}.

\bibitem[Audibert and Bubeck(2009)]{audibert09minimax}
Jean-Yves Audibert and Sebastien Bubeck.
\newblock Minimax policies for adversarial and stochastic bandits.
\newblock In \emph{Proceedings of the 22nd Annual Conference on Learning
  Theory}, 2009.

\bibitem[Audibert et~al.(2014)Audibert, Bubeck, and Lugosi]{audibert14regret}
Jean-Yves Audibert, Sebastien Bubeck, and Gabor Lugosi.
\newblock Regret in online combinatorial optimization.
\newblock \emph{Mathematics of Operations Research}, 39\penalty0 (1):\penalty0
  31--45, 2014.

\bibitem[Auer et~al.(2002{\natexlab{a}})Auer, Cesa-Bianchi, and
  Fischer]{auer02finitetime}
Peter Auer, Nicolo Cesa-Bianchi, and Paul Fischer.
\newblock Finite-time analysis of the multiarmed bandit problem.
\newblock \emph{Machine Learning}, 47:\penalty0 235--256, 2002{\natexlab{a}}.

\bibitem[Auer et~al.(2002{\natexlab{b}})Auer, Cesa-Bianchi, Freund, and
  Schapire]{AuCBFSch02}
Peter Auer, Nicolo Cesa-Bianchi, Yoav Freund, and Robert~E. Schapire.
\newblock The nonstochastic multiarmed bandit problem.
\newblock \emph{{SIAM} Journal of Computing}, 32\penalty0 (1):\penalty0 48--77,
  2002{\natexlab{b}}.

\bibitem[Bertsimas and Tsitsiklis(1997)]{bertsimas97introduction}
Dimitris Bertsimas and John Tsitsiklis.
\newblock \emph{Introduction to Linear Optimization}.
\newblock Athena Scientific, Belmont, MA, 1997.

\bibitem[Cesa-Bianchi and Lugosi(2012)]{cesabianchi12combinatorial}
Nicolo Cesa-Bianchi and Gabor Lugosi.
\newblock Combinatorial bandits.
\newblock \emph{Journal of Computer and System Sciences}, 78\penalty0
  (5):\penalty0 1404--1422, 2012.

\bibitem[Chen et~al.(2013)Chen, Wang, and Yuan]{chen13combinatorial}
Wei Chen, Yajun Wang, and Yang Yuan.
\newblock Combinatorial multi-armed bandit: General framework, results and
  applications.
\newblock In \emph{Proceedings of the 30th International Conference on Machine
  Learning}, pages 151--159, 2013.

\bibitem[Choi et~al.(2004)Choi, Moon, Zhang, Papagiannaki, and
  Diot]{choi04analysis}
Baek-Young Choi, Sue Moon, Zhi-Li Zhang, Konstantina Papagiannaki, and
  Christophe Diot.
\newblock Analysis of point-to-point packet delay in an operational network.
\newblock In \emph{Proceedings of the 23rd Annual Joint Conference of the IEEE
  Computer and Communications Societies}, 2004.

\bibitem[Dani et~al.(2008)Dani, Hayes, and Kakade]{dani08stochastic}
Varsha Dani, Thomas Hayes, and Sham Kakade.
\newblock Stochastic linear optimization under bandit feedback.
\newblock In \emph{Proceedings of the 21st Annual Conference on Learning
  Theory}, pages 355--366, 2008.

\bibitem[Edmonds(1970)]{edmonds70submodular}
Jack Edmonds.
\newblock Submodular functions, matroids, and certain polyhedra.
\newblock In \emph{Combinatorial Structures and Their Applications: Proceedings
  of the Calgary International Conference on Combinatorial Structures and Their
  Applications}, pages 69--87. 1970.

\bibitem[Fujishige(2005)]{fujishige05submodular}
Satoru Fujishige.
\newblock \emph{Submodular Functions and Optimization}.
\newblock Elsevier, Amsterdam, The Netherlands, 2005.

\bibitem[Gabillon et~al.(2013)Gabillon, Kveton, Wen, Eriksson, and
  Muthukrishnan]{gabillon13adaptive}
Victor Gabillon, Branislav Kveton, Zheng Wen, Brian Eriksson, and
  S.~Muthukrishnan.
\newblock Adaptive submodular maximization in bandit setting.
\newblock In \emph{Advances in Neural Information Processing Systems 26}, pages
  2697--2705, 2013.

\bibitem[Gabillon et~al.(2014)Gabillon, Kveton, Wen, Eriksson, and
  Muthukrishnan]{gabillon14largescale}
Victor Gabillon, Branislav Kveton, Zheng Wen, Brian Eriksson, and
  S.~Muthukrishnan.
\newblock Large-scale optimistic adaptive submodularity.
\newblock In \emph{Proceedings of the 28th AAAI Conference on Artificial
  Intelligence}, 2014.

\bibitem[Gai et~al.(2012)Gai, Krishnamachari, and Jain]{gai12combinatorial}
Yi~Gai, Bhaskar Krishnamachari, and Rahul Jain.
\newblock Combinatorial network optimization with unknown variables:
  Multi-armed bandits with linear rewards and individual observations.
\newblock \emph{IEEE/ACM Transactions on Networking}, 20\penalty0 (5):\penalty0
  1466--1478, 2012.

\bibitem[Guillory and Bilmes(2011)]{guillory11online}
Andrew Guillory and Jeff Bilmes.
\newblock Online submodular set cover, ranking, and repeated active learning.
\newblock In \emph{Advances in Neural Information Processing Systems 24}, pages
  1107--1115, 2011.

\bibitem[Kveton et~al.(2014{\natexlab{a}})Kveton, Wen, Ashkan, Eydgahi, and
  Eriksson]{kveton14matroid}
Branislav Kveton, Zheng Wen, Azin Ashkan, Hoda Eydgahi, and Brian Eriksson.
\newblock Matroid bandits: Fast combinatorial optimization with learning.
\newblock In \emph{Proceedings of the 30th Conference on Uncertainty in
  Artificial Intelligence}, pages 420--429, 2014{\natexlab{a}}.

\bibitem[Kveton et~al.(2014{\natexlab{b}})Kveton, Wen, Ashkan, and
  Szepesvari]{kveton14tight}
Branislav Kveton, Zheng Wen, Azin Ashkan, and Csaba Szepesvari.
\newblock Tight regret bounds for stochastic combinatorial semi-bandits.
\newblock \emph{CoRR}, abs/1410.0949, 2014{\natexlab{b}}.

\bibitem[Lai and Robbins(1985)]{lai85asymptotically}
T.~L. Lai and Herbert Robbins.
\newblock Asymptotically efficient adaptive allocation rules.
\newblock \emph{Advances in Applied Mathematics}, 6\penalty0 (1):\penalty0
  4--22, 1985.

\bibitem[Lam and Herlocker(2013)]{movielens}
Shyong Lam and Jon Herlocker.
\newblock {MovieLens 1M Dataset}.
\newblock http://www.grouplens.org/node/12, 2013.

\bibitem[Megiddo(1974)]{megiddo74optimal}
Nimrod Megiddo.
\newblock Optimal flows in networks with multiple sources and sinks.
\newblock \emph{Mathematical Programming}, 7\penalty0 (1):\penalty0 97--107,
  1974.

\bibitem[Neu and Bartok(2013)]{neu13efficient}
Gergely Neu and Gabor Bartok.
\newblock An efficient algorithm for learning with semi-bandit feedback.
\newblock In \emph{24th International Conference on Algorithmic Learning
  Theory}, volume 8139 of \emph{Lecture Notes in Computer Science}, pages
  234--248. 2013.

\bibitem[Oliveira and Pardalos(2005)]{oliveira05survey}
Carlos Oliveira and Panos Pardalos.
\newblock A survey of combinatorial optimization problems in multicast routing.
\newblock \emph{Computers and Operations Research}, 32\penalty0 (8):\penalty0
  1953--1981, 2005.

\bibitem[Papadimitriou and Steiglitz(1998)]{papadimitriou98combinatorial}
Christos Papadimitriou and Kenneth Steiglitz.
\newblock \emph{Combinatorial Optimization}.
\newblock Dover Publications, Mineola, NY, 1998.

\bibitem[Spring et~al.(2004)Spring, Mahajan, and Wetherall]{spring04measuring}
Neil Spring, Ratul Mahajan, and David Wetherall.
\newblock Measuring {ISP} topologies with {Rocketfuel}.
\newblock \emph{IEEE / ACM Transactions on Networking}, 12\penalty0
  (1):\penalty0 2--16, 2004.

\bibitem[Suehiro et~al.(2013)Suehiro, Hatano, Kijima, Takimoto, and
  Nagano]{suehiro12online}
Daiki Suehiro, Kohei Hatano, Shuji Kijima, Eiji Takimoto, and Kiyohito Nagano.
\newblock Online prediction under submodular constraints.
\newblock In \emph{23rd International Conference on Algorithmic Learning
  Theory}, volume 7568 of \emph{Lecture Notes in Computer Science}, pages
  260--274. 2013.

\bibitem[Thompson(1933)]{thompson33likelihood}
William.~R. Thompson.
\newblock On the likelihood that one unknown probability exceeds another in
  view of the evidence of two samples.
\newblock \emph{Biometrika}, 25\penalty0 (3-4):\penalty0 285--294, 1933.

\bibitem[Wen et~al.(2013)Wen, Kveton, Eriksson, and
  Bhamidipati]{wen13sequential}
Zheng Wen, Branislav Kveton, Brian Eriksson, and Sandilya Bhamidipati.
\newblock Sequential {Bayesian} search.
\newblock In \emph{Proceedings of the 30th International Conference on Machine
  Learning}, pages 977--983, 2013.

\bibitem[Wen et~al.(2014)Wen, Ashkan, Eydgahi, and Kveton]{WenAEK14}
Zheng Wen, Azin Ashkan, Hoda Eydgahi, and Branislav Kveton.
\newblock Efficient learning in large-scale combinatorial semi-bandits.
\newblock \emph{CoRR}, abs/1406.7443, 2014.

\bibitem[Whitney(1935)]{whitney35abstract}
Hassler Whitney.
\newblock On the abstract properties of linear dependence.
\newblock \emph{American Journal of Mathematics}, 57\penalty0 (3):\penalty0
  509--533, 1935.

\bibitem[Yue and Guestrin(2011)]{yue11linear}
Yisong Yue and Carlos Guestrin.
\newblock Linear submodular bandits and their application to diversified
  retrieval.
\newblock In \emph{Advances in Neural Information Processing Systems 24}, pages
  2483--2491, 2011.

\end{thebibliography}


\clearpage
\appendix

\section{Technical Lemmas}
\label{sec:lemmas}

\begin{lemma}
\label{lem:greedy vertices} Let $M = (E, f)$ be a polymatroid, $V$ be the vertices of base polyhedron $B_M$ in \eqref{eq:base polyhedron}, and $\Theta$ be the feasible solutions in \eqref{eq:feasible set}. Then $V = \Theta$.
\end{lemma}
\begin{proof}
The key observation is that $B_M$ is a convex polytope because it is an intersection of convex polytope $P_M$ \eqref{eq:independence polyhedron} and hyperplane $\sum_{e \in E} \bx(e) = K$. Therefore, any vector $\bx \in B_M$ is a convex combination of $V$. We prove $V = \Theta$ by proving that $V \subseteq \Theta$ and $\Theta \subseteq V$.

First, we prove that $V \subseteq \Theta$. By contradiction, suppose that there exists $\bx \in V$ such that $\bx \notin \Theta$. Since $\bx$ is a vertex of a convex polytope, there must exist a weight vector $\bw$ such that $\bx$ is a unique optimum in \eqref{eq:optimal}. By the definition of $\greedy$, $\bx = \greedy(M, \bw)$ and therefore $\bx \in \Theta$. This is clearly a contradiction.

Second, we prove that $\Theta \subseteq V$. By contradiction, suppose that there exists $\bx \in \Theta$ such that $\bx \notin V$. Since $B_M$ is a convex polytope, the solution $\bx$ can be expressed as a convex combination of the vertices in $V$. For simplicity of exposition, suppose that $\bx = \alpha \bx_1 + (1 - \alpha) \bx_2$, where $\set{\bx_1, \bx_2} \subset V$ and $\alpha \in (0, 1)$. Let $e_i$ be the first item in $\greedy$ where $\bx_1(e_i) < \bx_2(e_i)$. Since $\bx$ is generated by $\greedy$, and $f$ is a monotonic and submodular function, $\bx(e_i) \geq \bx_2(e_i)$. This is clearly a contradiction since $\bx(e_i) < (1 - \alpha) \bx_2(e_i)$. The case where $\bx_1(e_i) > \bx_2(e_i)$ is proved similarly. Finally, suppose that $\bx(e_i) \neq \bx_2(e_i)$ does not happen for any $e_i$. Then $\bx_1 = \bx_2$, which is also a contradiction.
\end{proof}

\begin{lemma}
\label{lem:pulls} For all items $e$ and $e^\ast \leq \rho(e)$:
\begin{align*}
  \E{\sum_{t = 1}^n \delta_t(e, e^\ast) \I{T_{t - 1}(e) > \ell}}{\bw_1, \dots, \bw_n} \leq
  \frac{4}{3} \pi^2
\end{align*}
when $\ell = \floors{\frac{8}{\Delta_{e, e^\ast}^2} \log n}$.
\end{lemma}
\begin{proof}
First, we note that $\delta_t(e, e^\ast) \leq 1$. Moreover, by Theorem~\ref{thm:regret decomposition}, the event $\delta_t(e, e^\ast) > 0$ implies that we observe the weight of item $e$ and $U_t(e) \geq U_t(e^\ast)$. Based on these facts, it follows that:
\begin{align}
  \sum_{t = 1}^n \delta_t(e, e^\ast) \I{T_{t - 1}(e) > \ell}
  & \leq \sum_{t = 1}^n \I{\delta_t(e, e^\ast) > 0, \ T_{t - 1}(e) > \ell} \nonumber \\
  & \leq \sum_{t = \ell + 1}^n \I{U_t(e) \geq U_t(e^\ast), \ T_{t - 1}(e) > \ell} \nonumber \\
  & \leq \sum_{t = \ell + 1}^n \sum_{s = 1}^t \sum_{s_e = \ell + 1}^t
  \I{\hat{\bw}_{s_e}(e) + c_{t - 1, s_e} \geq \hat{\bw}_s(e^\ast) + c_{t - 1, s}} \nonumber \\
  & = \sum_{t = \ell}^{n - 1} \sum_{s = 1}^{t + 1} \sum_{s_e = \ell + 1}^{t + 1}
  \I{\hat{\bw}_{s_e}(e) + c_{t, s_e} \geq \hat{\bw}_s(e^\ast) + c_{t, s}}.
\end{align}
When $\hat{\bw}_{s_e}(e) + c_{t, s_e} \geq \hat{\bw}_s(e^\ast) + c_{t, s}$, at least one of the following events must happen:
\begin{align}
  \hat{\bw}_s(e^\ast) & \leq \bar{\bw}(e^\ast) - c_{t, s} \label{eq:H1} \\
  \hat{\bw}_{s_e}(e) & \geq \bar{\bw}(e) + c_{t, s_e} \label{eq:H2} \\
  \bar{\bw}(e^\ast) & < \bar{\bw}(e) + 2 c_{t, s_e}. \label{eq:minimum pulls}
\end{align}
We bound the probability of the first two events, \eqref{eq:H1} and \eqref{eq:H2}, using Hoeffding's inequality:
\begin{align}
  P(\hat{\bw}_s(e^\ast) \leq \bar{\bw}(e^\ast) - c_{t, s})
  & \leq \exp[-4 \log t] = t^{-4} \\
  P(\hat{\bw}_{s_e}(e) \geq \bar{\bw}(e) + c_{t, s_e})
  & \leq \exp[-4 \log t] = t^{-4}.
\end{align}
When $s_e \geq \frac{8}{\Delta_{e, e^\ast}^2} \log n$, the third event \eqref{eq:minimum pulls} cannot happen because:
\begin{align}
  \bar{\bw}(e^\ast) - \bar{\bw}(e) - 2 c_{t, s_e} =
  \Delta_{e, e^\ast} - 2 \sqrt{\frac{2 \log t}{s_e}} \geq
  0.
\end{align}
This is guaranteed when $\ell = \floors{\frac{8}{\Delta_{e, e^\ast}^2} \log n}$. Finally, we combine all of our claims and get:
\begin{align}
  \E{\sum_{t = 1}^n \delta_t(e, e^\ast) \I{T_{t - 1}(e) > \ell}}{\bw_1, \dots, \bw_n}
  & \leq \sum_{t = \ell}^{n - 1} \sum_{s = 1}^{t + 1} \sum_{s_e = \ell + 1}^{t + 1}
  [P(\hat{\bw}_s(e^\ast) \leq \bar{\bw}(e^\ast) - c_{t, s}) + {} \nonumber \\
  & \hspace{1.06in} P(\hat{\bw}_{s_e}(e) \geq \bar{\bw}(e) + c_{t, s_e})] \nonumber \\
  & \leq \sum_{t = 1}^\infty 2 (t + 1)^2 t^{-4} \nonumber \\
  & \leq \sum_{t = 1}^\infty 8 t^{-2} \nonumber \\
  & = \frac{4}{3} \pi^2.
\end{align}
The last equality follows from the fact that $\displaystyle \sum_{t = 1}^\infty t^{-2} = \frac{\pi^2}{6}$.
\end{proof}

\begin{lemma}[\cite{kveton14matroid}]
\label{lem:multiple optimal pulls} Let $\Delta_1 \geq \ldots \geq \Delta_K$ be a sequence of $K$ positive numbers. Then:
\begin{align*}
  \left[\Delta_1 \frac{1}{\Delta_1^2} + \sum_{k = 2}^K \Delta_k
  \left(\frac{1}{\Delta_k^2} - \frac{1}{\Delta_{k - 1}^2}\right)\right] \leq
  \frac{2}{\Delta_K}.
\end{align*}
\end{lemma}
\begin{proof}
First, we note that:
\begin{align}
  \left[\Delta_1 \frac{1}{\Delta_1^2} + \sum_{k = 2}^K \Delta_k
  \left(\frac{1}{\Delta_k^2} - \frac{1}{\Delta_{k - 1}^2}\right)\right] =
  \sum_{k = 1}^{K - 1} \frac{\Delta_k - \Delta_{k + 1}}{\Delta_k^2} + \frac{1}{\Delta_K}.
\end{align}
Second, by our assumption, $\Delta_k \geq \Delta_{k + 1}$ for all $k < K$. Therefore:
\begin{align}
  \sum_{k = 1}^{K - 1} \frac{\Delta_k - \Delta_{k + 1}}{\Delta_k^2} + \frac{1}{\Delta_K}
  & \leq \sum_{k = 1}^{K - 1} \frac{\Delta_k - \Delta_{k + 1}}{\Delta_k \Delta_{k + 1}} +
  \frac{1}{\Delta_K} \nonumber \\
  & = \sum_{k = 1}^{K - 1} \left[\frac{1}{\Delta_{k + 1}} - \frac{1}{\Delta_k}\right] +
  \frac{1}{\Delta_K} \nonumber \\
  & = \frac{2}{\Delta_K} - \frac{1}{\Delta_1} \nonumber \\
  & < \frac{2}{\Delta_K}.
\end{align}
This concludes our proof.
\end{proof}

\end{document}